\documentclass{article}

\usepackage[utf8]{inputenc} % allow utf-8 input
\usepackage[T1]{fontenc}    % use 8-bit T1 fonts
\usepackage{booktabs}       % professional-quality tables
\usepackage{nicefrac}       % compact symbols for 1/2, etc.
\usepackage{microtype}      % microtypography
\usepackage{times}             % times font

\usepackage[round]{natbib}
\usepackage{amssymb,amsmath,amsthm,bbm}
\usepackage[margin=1in]{geometry}
\usepackage{verbatim,float,url,dsfont}
\usepackage{graphicx,subfigure,psfrag}
\usepackage{algorithm,algorithmic}
\usepackage{mathtools,enumitem}
\usepackage[colorlinks=true,citecolor=blue,urlcolor=blue,linkcolor=blue]{hyperref}
\usepackage{boldline,multirow}
\usepackage{tcolorbox}
\usepackage{breqn}

\newtheorem{theorem}{Theorem}
\newtheorem{lemma}{Lemma}
\newtheorem{corollary}{Corollary}
\newtheorem{proposition}{Proposition}
\theoremstyle{definition}
\newtheorem{remark}{Remark}

\newtheorem{example}{Example}

\DeclarePairedDelimiter\ceil{\lceil}{\rceil}
\DeclarePairedDelimiter\floor{\lfloor}{\rfloor}

\usepackage{trackchanges}
\addeditor{db}
\addeditor{yw}

\usepackage{accents}
\newcommand{\ubar}[1]{\underaccent{\bar}{#1}}

\renewcommand{\st}{\mathop{\mathrm{subject\,\,to}}}

\def\R{\mathbb{R}}
\def\E{\mathbb{E}}

\def\Var{\mathrm{Var}}

\def\cA{\mathcal{A}}

\def\cF{\mathcal{F}}

\def\cH{\mathcal{H}}

\def\cN{\mathcal{N}}
\def\cP{\mathcal{P}}
\def\cS{\mathcal{S}}
\def\cT{\mathcal{T}}

\def\cX{\mathcal{X}}

\def\TV{\mathrm{TV}}

\newcommand{\red}[1]{\textcolor{red}{#1}}

\usepackage[utf8]{inputenc} % allow utf-8 input
\usepackage[T1]{fontenc}    % use 8-bit T1 fonts
\usepackage{booktabs}       % professional-quality tables
\usepackage{amsfonts}       % blackboard math symbols
\usepackage{nicefrac}       % compact symbols for 1/2, etc.
\usepackage{microtype}      % microtypography

\usepackage{times}             % times font
\usepackage{amsthm}

\usepackage{bbm}
\usepackage{verbatim,float,dsfont}
\usepackage{psfrag}
\usepackage{algorithm,algorithmic}
\usepackage{mathtools,enumitem}
\mathtoolsset{showonlyrefs}
\usepackage{tcolorbox}
\usepackage{breqn,lipsum}
\usepackage{thmtools,thm-restate}
\usepackage{graphicx,stackengine}
\usepackage{array}
\usepackage{caption}

\def\shownotes{1}  %set 1 to show author notes
\ifnum\shownotes=1
\newcommand{\authnote}[2]{$\ll$\textsf{\footnotesize #1 notes: #2}$\gg$}
\else
\newcommand{\authnote}[2]{}
\fi

\newcommand{\yw}[1]{{\color{violet}\authnote{Yu-Xiang}{#1}}}

\usepackage{accents}

\newcommand{\ARROWS}{\textsc{Arrows}}

\newcommand{\lsim}{\raisebox{-0.13cm}{~\shortstack{$<$ \\[-0.07cm]
      $\sim$}}~}

\def\R{\mathbb{R}}
\def\E{\mathbb{E}}

\def\Var{\mathrm{Var}}

\def\cA{\mathcal{A}}

\def\cF{\mathcal{F}}

\def\cH{\mathcal{H}}

\def\cN{\mathcal{N}}
\def\cP{\mathcal{P}}
\def\cS{\mathcal{S}}
\def\cT{\mathcal{T}}

\def\cX{\mathcal{X}}

\def\TV{\mathrm{TV}}

\title{Online Forecasting of Total-Variation-bounded Sequences}  
\author{Dheeraj Baby \\dheeraj@ucsb.edu \and Yu-Xiang Wang \\yuxiangw@cs.ucsb.edu}

\date{UC Santa Barbara}

\begin{document}

\maketitle

\begin{abstract}
	We consider the problem of online forecasting of sequences of length $n$ with total-variation at most $C_n$ using observations contaminated by independent $\sigma$-subgaussian noise. We design an $O(n\log n)$-time algorithm that achieves a cumulative square error of $\tilde{O}(n^{1/3}C_n^{2/3}\sigma^{4/3} + C_n^2)$ with high probability. 
	We also prove a lower bound that matches the upper bound in all parameters (up to a $\log(n)$ factor). %The result is rate-optimal as it matches the known minimax rate for the offline nonparametric estimation of the same class \citep{locadapt}. 
	To the best of our knowledge, this is the first \emph{polynomial-time} algorithm that achieves the optimal $O(n^{1/3})$ rate in forecasting total variation bounded sequences and the first algorithm that \emph{adapts to unknown} $C_n$.
	Our proof techniques leverage the special localized structure of Haar wavelet basis and the adaptivity to unknown smoothness parameters in the classical wavelet smoothing \citep{donoho1998minimax}.
	We also compare our model to the rich literature of dynamic regret minimization and nonstationary stochastic optimization, where our problem can be treated as a special case. We show that the workhorse in those settings --- online gradient descent and its variants with a fixed restarting schedule --- are instances of a class of \emph{linear forecasters} that require a suboptimal regret of $\tilde{\Omega}(\sqrt{n})$. This implies that the use of more adaptive algorithms is necessary to obtain the optimal rate. 
\end{abstract}

%Pure text version of the abstract (to use in the CMT submission)

% 	We consider the problem of online forecasting of sequences of length $n$ with total-variation at most $C_n$ using observations contaminated by independent $\sigma$-subgaussian noise. We design an $O(n\log n)$-time algorithm that achieves a cumulative square error of $\tilde{O}(n^{1/3}C_n^{2/3}\sigma^{4/3} + C_n^2)$ with high probability. We also prove a lower bound that matches the upper bound in all parameters (up to a $\log(n)$ factor). To the best of our knowledge, this is the first **polynomial-time** algorithm that achieves the optimal $O(n^{1/3})$ rate in forecasting total variation bounded sequences and the first algorithm that **adapts to unknown** $C_n$.Our proof techniques leverage the special localized structure of Haar wavelet basis and the adaptivity to unknown smoothness parameters in the classical wavelet smoothing [Donoho et al., 1998]. We also compare our model to the rich literature of dynamic regret minimization and nonstationary stochastic optimization, where our problem can be treated as a special case. We show that the workhorse in those settings --- online gradient descent and its variants with a fixed restarting schedule --- are instances of a class of **linear forecasters** that require a suboptimal regret of $\tilde{\Omega}(\sqrt{n})$. This implies that the use of more adaptive algorithms is necessary to obtain the optimal rate. 

\section{Introduction}

Nonparametric regression is a fundamental class of problems that has been studied for more than half a century in statistics and machine learning
 %has a rich history in statistics and machine learning, carrying well over 50 years of associated literature
  \citep{nadaraya1964estimating,de1978practical,wahba1990spline,donoho1998minimax,mallat1999wavelet,scholkopf2001learning,rasmussen2006gaussian}.
It solves the following problem:
\begin{itemize}
	\item Let $y_i = f(u_i) + \text{Noise}$ for $i=1,...,n$. How can we estimate a  function $f$ using data points $(u_1,y_1),...,(u_n,y_n)$  and 
	%in conjunction with 
	the knowledge that $f$  belongs to a function class $\cF$?
\end{itemize}
%of estimating a smooth function using noisy observations of the function values without making assumptions about the form or shape of the function. 
Function class $\cF$ typically imposes only weak regularity assumptions on the function $f$ such as boundedness and smoothness, which makes nonparametric regression widely applicable to many real-life applications especially those with unknown physical processes.

A recent and successful class of nonparametric regression technique called trend filtering 
\citep{hightv,l1tf,tibshirani2014adaptive,wang2014falling} was shown to have the property of \emph{local adaptivity} \citep{locadapt} in both theory and practice. We say a nonparametric regression technique is \emph{locally adaptive} if it can cater to local differences in smoothness, hence allowing more accurate estimation of functions with varying smoothness and abrupt changes. %heterogeneous smoothness. 
For example, for functions with bounded total variation (when $\cF$ is a total variation class), standard nonparametric regression techniques such as kernel smoothing and smoothing splines have a mean square error (MSE) of $O(n^{-1/2})$ while trend filtering has the optimal $O(n^{-2/3})$.

Trend filtering is, however, a batch learning algorithm where one observes the entire dataset ahead of the time and makes inference about the past.  This makes it inapplicable to the many time series problems that motivate the study of trend filtering in the first place \citep{l1tf}. These include influenza forecasting, inventory planning, economic policy-making, financial market prediction and so on. In particular, it is unclear whether the advantage of trend filtering methods in estimating functions with heterogeneous smoothness (e.g., sharp changes) would carry over to the online forecasting setting. The focus of this work is in developing theory and algorithms for locally adaptive online forecasting which predicts the immediate future value of a function with heterogeneous smoothness using only noisy observations from the past.

\subsection{Problem Setup}
\label{sec:problem_setup}
\vspace{-1em}
\begin{figure}[h!]
	\centering
	\fbox{
		\begin{minipage}{13 cm}
			\begin{enumerate}
				\setlength\itemsep{0em}
				\item Fix action time intervals $1,2,...,n$
				\item The player declares a forecasting strategy $\cA_i:  \R^{i-1}\rightarrow \R$ for $i=1,...,n$.
				\item An adversary chooses a sequence $\theta_{1:n} = [\theta_1, \theta_2,\hdots, \theta_n]^T \in \R^n$. %\theta_i \text{ for  }i=1,...,n]$. % and $f:=[f_1,...,f_n]\in\R^n$.
				\item For every time point $i = 1,...,n$:
				\begin{enumerate}
					\item We play $x_i =  \cA_i(y_1,...,y_{i-1})$.
					\item We receive a feedback $y_i = \theta_i +  Z_i$, where $Z_i$ is a zero-mean, independent subgaussian noise.
				\end{enumerate}
				\item At the end, the player suffers a cumulative error %regret of
				$\sum_{i=1}^n \big(x_i- \theta_i\big)^2$.
			\end{enumerate}
		\end{minipage}
	}
	\caption{\emph{Nonparametric online forecasting model. The focus of the proposed work is to design a forecasting strategy that minimizes the expected cumulative square error. %regret. 
	Note that the problem depends a lot on the choice of the sequence $\theta_i$.  Our primary interest is on sequences with bounded total variation (TV) so that $\sum_{i=2}^{n}|\theta_i - \theta_{i-1}| \le C_n$, but we will also talk about the adaptivity our our method to easier problems such as forecasting Sobolev and Holder functions.}}
	%				 The attained minimum value is called the minimax regret.  Note that the $\cF$ and in particular  }
	\label{fig:nonpara_online_forecasting}
\end{figure}We propose a model for nonparametric online forecasting as described in Figure \ref{fig:nonpara_online_forecasting}. This model can be re-framed in the language of the online convex optimization model with three differences. 
\begin{enumerate}
\setlength{\parskip}{0pt}
	\item We consider only quadratic loss functions of the form $\ell_t(x) = (x -\theta_t)^2$.
	\item The learner receives independent \emph{noisy} gradient feedback, rather than the exact gradient.
	\item The criterion of interest is redefined as the \emph{dynamic regret} \citep{zinkevich2003online,besbes2015non}:
\begin{equation}\label{eq:dynamic_regret}
	R_{\text{dynamic}}(\cA,\ell_{1:n}) := \E \left[ \sum_{t=1}^n \ell_t(x_t) \right] - %\inf_{x_{1:n}\in \Theta} 
	\sum_{t=1}^n \inf_{x_t }\ell_t(x_t).
\end{equation}
\end{enumerate} 
The new criterion is called a dynamic regret because we are now comparing to a stronger dynamic baseline that chooses an optimal $x$ in every round.
Of course in general, the dynamic regret will be linear in $n$ \citep{jadbabaie2015online}. 
To make the problem non-trivial, we restrict our attention to sequences of $\ell_1,...,\ell_n$ that are \emph{regular}, which makes it possible to design algorithms with \emph{sublinear} dynamic regret.
%\citet{zinkevich2003online} imposes additional constraints on $\Theta$ --- the sequence of comparators, while we keep $\Theta = \R^n$ and impose regularity assumptions on $\ell_1,...,\ell_n$. %as in \citep{besbes2015non}. 
In particular, we borrow ideas from the nonparametric regression literature and consider sequences $[\theta_1,...,\theta_n]$ that are discretizations of functions in the continuous domain. Regularity assumptions emerge naturally as we consider canonical functions classes such as the Holder class, Sobolev class and Total Variation classes \citep[see, e.g.,][for a review]{tsybakov_book}.
%such as those in the Holder class, Sobolev class and Total Variation class . 
%This is a strictly harder problem than the problem of nonparametric regression because the latter also uses observations from the future.

%Regularity assumptions emerge naturally as we consider canonical functions classes such as the Holder class, Sobolev class and Total Variation classes.
%within canonical function classes such as Holder classes, Sobolev classes and Total Variation classes. 
%Our proposed model borrows ideas from the nonparametric regression world and imposes additional regularity assumptions to these functions $\ell_1,...,\ell_n$. 
%This choice makes it possible for us to design algorithms that have sublinear dynamic regret on these problems. 

%Dynamic regret is an important and more refined criterion than the static regret.

\subsection{Assumptions}
We consolidate all the assumptions used in this work and provide necessary justifications for them. 
\begin{description}[font=$\bullet$\scshape\bfseries]
\item (A1) The time horizon for the online learner is known to be $n$.
\item (A2)  The parameter $\sigma^2$ of subgaussian noise in the observations is known.
\item (A3) The ground truth denoted by $\theta_{1:n} = [\theta_1,..., \theta_n]^T$ has its total variation bounded by some positive $C_n$, i.e., we take $\cF$ to be the total variation class $\TV(C_n) := \{ \theta_{1:n}\in \R^n :   \|D \theta_{1:n}\|_1 \leq C_n \}$ where  $D$ is the discrete difference operator. Here $D \theta_{1:n} = [\theta_2 - \theta_1, \ldots, \theta_n - \theta_{n-1}]^T$.
 %We call the class of such functions the total variation class, denoted as $\TV(C_n) := \{ \btheta\in \R^n |   \|D \btheta\|_1 \leq C_n \}$. Here $D$ is the discrete difference operator. 
%This is denoted by $\|D\btheta\|_{1} \le C_n$ for some known positive $C_n$.   
\item (A4) $|\theta_1| \le U$.
\end{description}

The knowledge of $\sigma^2$ in assumption (A2) is primarily used to get the optimal dependence of $\sigma$ in minimax rate. This assumption can be relaxed in practice by using the Median Absolute Deviation estimator as described in Section 7.5 of \citet{DJBook} to estimate $\sigma^2$ robustly. Assumption (A3) features a samples from a large class of functions with spatially inhomogeneous degree of smoothness. The functions residing in this class need not even be continuous. Our goal is to propose a policy that is locally adaptive whose empirical mean squared error converges at the minimax rate for this function class.  We stress that we do \emph{not} assume that the learner knows $C_n$. The problem is open and nontrivial even when $C_n$ is known. Assumption (A4) is very mild as it puts restriction only to the first value of the sequence. This assumption controls the inevitable prediction error for the first point in the sequence.

\subsection{Our Results}
%\paragraph{Contributions}
The major contributions of this work are summarized below.
\begin{itemize}
    \item It is known that the minimax MSE for \emph{smoothing} sequences in the TV class is $\tilde{\Omega}(n^{-2/3})$. This implies a lowerbound of $\tilde{\Omega}(n^{1/3})$ for the dynamic regret in our setting. %As we discuss in Section~\ref{sec:related_work}, the existing algorithms only yield a dynamic regret of $\tilde{O}(\sqrt{n})$. 
    We present a policy \ARROWS{} (\textbf{A}daptive \textbf{R}estarting \textbf{R}ule for \textbf{O}nline averaging using \textbf{W}avelet \textbf{S}hrinkage)  with a nearly minimax dynamic regret $\tilde{O}(n^{1/3})$ and a run-time complexity of $O(n\log n)$.
    %As we discuss in Section~\ref{sec:related_work}, existing algorithms only yield a suboptimal $\tilde{O}(\sqrt{n})$ dynamic regret.
    \item We show that a class of forecasting strategies --- including the popular Online Gradient Descent (OGD) with fixed restarts \citep{besbes2015non}, moving averages (MA) \citep{box1970time} --- are fundamentally limited by $\tilde{\Omega}(\sqrt{n})$ regret.
    \item 
We also provide a more refined lower bound that characterized the dependence of $U,C_n$ and $\sigma$, which certifies the adaptive optimality of \ARROWS{} in all regimes. The bound also reveals a subtle price to pay when we move from the smoothing problem to the forecasting problem, which indicates the separation of the two problems when $C_n/\sigma \gg n^{1/4}$, a regime where the forecasting problem is \emph{strictly} harder (See Figure~\ref{fig:minimax_rate}).
\item Lastly, we consider forecasting sequences in Sobolev classes and Holder classes and establish that \ARROWS{} can automatically \emph{adapt} to the optimal regret of these \emph{simpler} function classes as well, while OGD and MA cannot, unless we change their tuning parameter (to behave suboptimally on the TV class).
%    \item Naive implementation of our policy will have a run-time complexity of $O(n^2)$. We exploit the sequential structure of our policy and sparsity in wavelet transforms to construct an $O(n\log(n))$ implementation.
\end{itemize}

\section{Related Work}
\label{sec:related_work}
The topic of this paper sits well in between two amazing bodies of literature: nonparametric regression and online learning. Our results therefore contribute to both fields and hopefully will inspire more interplay between the two communities. Throughout this paper when we refer $\tilde{O}(n^{1/3})$ as the optimal regret, we assume the parameters of the problem are such that it is acheivable (see Figure~\ref{fig:minimax_rate}).

\noindent\textbf{Nonparametric regression.}
As we mentioned before, our problem --- online nonparametric forecasting --- is motivated by the idea of using locally adaptive nonparametric regression for time series forecasting \citep{locadapt,l1tf,tibshirani2014adaptive}. It is more challenging than standard nonparametric regression because we do not have access to the data in the future.  While our proof techniques make use of several components (e.g., universal shrinkage) from the seminal work in wavelet smoothing \citep{donoho1990minimax,donoho1998minimax}, the way we use them to construct and analyze our algorithm is new and more generally applicable for converting non-parametric regression methods to forecasting methods.

\noindent\textbf{Adaptive Online Learning.}
Our problem is also connected to a growing literature on adaptive online learning which aims at matching the performance of a stronger time-varying baseline \citep{zinkevich2003online,hall2013dynamical,besbes2015non,chen2018non,jadbabaie2015online,hazan2007adaptive,daniely2015strongly,yang2016tracking,zhang2018adaptive,zhang2018dynamic,chen2018smoothed}.
Many of these settings are highly general and we can apply their algorithms directly to our problem, but to the best of our knowledge, none of them achieves the optimal $\tilde{O}(n^{1/3})$ dynamic regret.

In the remainder of this section, we focus our discussion on how to apply the regret bounds in non-stationary stochastic optimization \citep{besbes2015non,chen2018non} to our problem while leaving more elaborate discussion with respect to alternative models (e.g. the constrained comparator approach \citep{zinkevich2003online,hall2013dynamical}, adaptive regret \citep{jadbabaie2015online,zhang2018adaptive}, competitive ratio \citep{bansal20152,chen2018smoothed}), as well as the comparison to the classical time series models to Appendix~\ref{app:related}. 

%\citet{zinkevich2003online,hall2013dynamical,zhang2018adaptive} restrict the sequence of comparators to be ``slowly changing''; \citet{hazan2007adaptive,daniely2015strongly,zhang2018adaptive} minimizes the ``adaptive regret'', which simultaneously minimizes the static regret for all contiguous subsequences; \citet{bansal20152,chen2018smoothed} adds a switching cost or a regularization cost to the regret itself. The closest to us is arguably the problem of non-stationary stochastic optimization \citep{besbes2015non,yang2016tracking,chen2018non}, which aims at minimizing the dynamic regret against an unrestricted pointwise optimal baseline while making variational assumptions about the sequence of functions to minimize. 

%
%
%While several of the proposed algorithms in this literature can be applied to our problem, none of them can be used to obtain the optimal $\tilde{O}(n^{1/3})$ dynamic regret. 

%The closest to us is argubaly 

%To the best of our knowledge, no existing algorithms in online learning can be used to obtain the optimal $\tilde{O}(n^{1/3})$ dynamic regret in our problem. 

%In this section, we describe the relationship to these prior work and discuss how to apply them to our problem.
%\subsection{Regret from Non-Stationary Stochastic Optimization}
%\label{NonStoch}
\noindent\textbf{Regret from Non-Stationary Stochastic Optimization}
The problem of non-stationary stochastic optimization is more general than our model because instead of considering only the quadratic functions, $\ell_t(x) = (x -\theta_t)^2$, they work with the more general class of strongly convex functions and general convex functions. They also consider both noisy gradient feedbacks (stochastic first order oracle) and noisy function value feedbacks (stochastic zeroth order oracle).

In particular, \citet{besbes2015non} define a quantity $V_n$ which captures the total amount of ``variation'' of the functions $\ell_{1:n}$ using
$
V_n  :=  \sum_{i=1}^{n-1}  \| \ell_{i+1} - \ell_{i}\|_{\infty}.
$ \footnote{The $V_n$ definition in \citep{besbes2015non} for strongly convex functions are defined a bit differently, the $\|\cdot\|_{\infty}$ is taken over the convex hull of minimizers. This creates some subtle confusions regarding our results which we explain in details in Appendix~\ref{app:lower}.}
\citet{chen2018non} generalize the notion to
%More generally,  in a more recent paper \citep{chen2018non}, the variation is generalized to
$
V_n(p,q)  := \left(\sum_{i=1}^{n-1}  \| \ell_{i+1} - \ell_{i}\|_{p}^q\right)^{1/q}
$
for any $1 \leq p,q \leq +\infty$ where $\|\cdot\|_p := (\int | \cdot(x)|^p dx)^{1/p}$ is the standard $L_p$ norm for functions\footnote{We define $V_n(p,q)$ to be a factor of $n^{-1/q}$ times bigger than the original scaling presented in \citep{chen2018non} so the results become comparable to that of \citep{besbes2015non}.}. Table~\ref{tab:known_results_in_nso} summarizes the known results under the non-stationary stochastic optimization setting.
\begin{table}[h!]
	\centering
	\caption{Summary of known minimax dynamic regret in the non-stationary stochastic optimization model. Note that the choice of $q$ does not affect the minimax rate in any way, but the choice of $p$ does.  ``-'' indicates that the no upper or lower bounds are known for that setting. }\label{tab:known_results_in_nso}
	\resizebox{\textwidth}{!}{
	\begin{tabular}{c|cc|cc}
		\hline
		& \multicolumn{2}{c}{Noisy gradient feedback} &  \multicolumn{2}{c}{Noisy function value feedback} \\\hline
	Assumptions on $\ell_{1:n}$ 	&$p=+\infty$& $1\leq p<+\infty$&$p=+\infty$ &$1\leq p<+\infty$\\\hline
		Convex \& Lipschitz & $\Theta(n^{2/3} V_n^{1/3})$ &$O(n^{\frac{2p+d}{3p+d}} V_n(p,q)^{\frac{p}{3p+d}})$ & - & -\\
		Strongly convex \& Smooth &$\Theta(n^{1/2}V_n^{1/2})$&   $\Theta(n^{\frac{2p+d}{4p+d}} V_n(p,q)^{\frac{2p}{4p+d}})$ & $\Theta(n^{2/3} V_n^{1/3})$&  $\Theta(n^{\frac{4p+d}{6p+d}}V_n(p,q)^{\frac{2p}{6p+d}})$  \\\hline
	\end{tabular}
	}
\end{table}

Our assumption on the underlying trend $\theta_{1:n} \in \cF$ can be used to construct an upper bound of this quantity of variation $V_n$ or $V_n(p,q)$. As a result, the algorithms in non-stationary stochastic optimization and their dynamic regret bounds in Table~\ref{tab:known_results_in_nso} will apply to our problem (modulo additional restrictions on bounded domain). However, our preliminary investigation suggests that this direct reduction does \emph{not}, in general, lead to optimal algorithms. We illustrate this observation in the following example.
\begin{example}\label{ex:nso_for_nof}
Let $\cF$ be the set of all bounded sequences in the total variation class $TV(1)$. 
	%Applying non-stationary stochastic optimization to nonparametric online forecasting problem with 's intersection with a boundedness constraint. 
	It can be worked out that $V_n(p,q) = O(1)$ for all $p,q$. Therefore the smallest regret from \citep{besbes2015non,chen2018non} is obtained by taking $p\rightarrow +\infty$, which gives us a regret of $O(n^{1/2})$.  Note that we expect the optimal regret to be $\tilde{O}(n^{1/3})$ according to the theory of locally adaptive nonparametric regression.
\end{example}

In Example~\ref{ex:nso_for_nof}, we have demonstrated that one cannot achieve  the optimal dynamic regret using known results in non-stationary stochastic optimization.  We show in section \ref{gen_start} that ``Restarting OGD'' algorithm has a fundamental lower bound of $\tilde{\Omega}(\sqrt{n})$ on dynamic regret in the TV class.

\noindent\textbf{Online nonparametric regression.} As we finalize our manuscript, it comes to our attention that our problem of interest in Figure~\ref{fig:nonpara_online_forecasting} can be cast as a special case of the ``online nonparametric regression'' problem \citep{rakhlin2014online,gaillard2015chaining}. The general result of \citet{rakhlin2014online} 
%shows that the minimax regret for the TV class is $\tilde{O}(n^{1/3})$ via a non-constructive argument (see more details in Appendix~\ref{app:related}). 
implies the \emph{existence} of an algorithm that enjoys a $\tilde{O}(n^{1/3})$ regret for the TV class without explicitly constructing one, which shows that $n^{1/3}$ is the minimax rate when $C_n= O(1)$ (see more details in Appendix~\ref{app:related}).
To the best of our knowledge, our proposed algorithm remains the first \emph{polynomial time} algorithm with $\tilde{O}(n^{1/3})$ regret and our results reveal more precise (optimal) upper and lower bounds on all parameters of the problem (see Section~\ref{sec:precise_lowerbound}).

%A more elaborate discussion on other related works is deferred to Appendix~\ref{app:related}.

% the algorithm that obtains the optimal dynamic regret for non-stationary stochastic optimization is in fact just online gradient descent with a carefully chosen fixed restarting schedule

%\subsection{Regret from Adaptive Optimistic Mirror Descent}
%\label{aomd}
\section{Main results}
We are now ready to present our main results. 
\subsection{Limitations of Linear Forecasters}
\label{gen_start}
Restarting OGD as discussed in Example 1, fails to achieve the optimal regret in our setting. A curious question to ask is whether it is the algorithm itself that fails or it is an artifact of a potentially suboptimal regret analysis. To answer this, let's consider the class of linear forecasters --- estimators that outputs a fixed linear transformation of the observations $y_{1:n}$. The following preliminary result shows that Restarting OGD is a linear forecaster . By the results of \citet{donoho1998minimax}, linear smoothers are fundamentally limited in their ability to estimate functions with heterogeneous smoothness. Since forecasting is harder than smoothing, this limitation gets directly translated to the setting of linear forecasters.

\begin{proposition}\label{prop:OGD_lowerbound}
	Online gradient descent with a fixed restart schedule is a linear forecaster. Therefore, it has a dynamic regret of at least $\tilde{\Omega}(\sqrt{n})$.
\end{proposition}
\begin{proof}
	First, observe that the stochastic gradient is of form $2(x_t - y_t)$ where $x_t$ is what the agent played at time $t$ and $y_t$ is the noisy observation $\theta_t + \text{Independent noise}$.   By the online gradient descent strategy with the fixed restart interval and an inductive argument, $x_t $ is a linear combination of $y_1,...,y_{t-1}$ for any $t$. Therefore, the entire vector of predictions $x_{1:t}$ is a fixed linear transformation of $y_{1:t-1}$.  The fundamental lower bound for linear smoothers from \citet{donoho1998minimax} implies that this algorithm will have a regret of at least $\tilde{\Omega}(\sqrt{n})$.
\end{proof}
The proposition implies that we will need fundamentally new \emph{nonlinear} algorithmic components to achieve the optimal $O(n^{1/3})$ regret, if it is achievable at all!

\subsection{Policy}
In this section, we present our policy \ARROWS{} (Adaptive Restarting Rule for Online averaging using Wavelet Shrinkage). The following notations are introduced for describing the algorithm.
\begin{description}[font=$\bullet$\scshape\bfseries]
\item $t_h$ denotes start time of the current bin and $t$ be the current time point.
\item $\bar{y}_{t_h:t}$ denotes the average of the $y$ values for time steps indexed from $t_h$ to $t$.
\item $pad_0(y_{t_h},...,y_{t})$ denotes the vector $(y_{t_h}-\bar{y}_{t_h:t},...,y_{t}-\bar{y}_{t_h:t})^T$ zero-padded at the end till its length is a power of 2. \emph{i.e}, a re-centered and padded version of observations.
\item $T(x)$ where $x$ is a sequence of values, denotes the element-wise soft thresholding of the sequence with threshold $\sigma\sqrt{\beta\log(n)}$
\item H denotes the orthogonal discrete Haar wavelet transform matrix of proper dimensions
\item Let $Hx = \alpha = [\alpha_1,\alpha_2,...,\alpha_k]^T$ where $k$ being a power of 2 is the length of $x$. Then the vector $[\alpha_2,...,\alpha_k]^T$ can be viewed as a concatenation of $\log_2 k$ contiguous blocks represented by $\alpha [l], l=0,...,\log_2(k) - 1$. Each block $\alpha [l]$ at level $l$ contains $2^{l}$ coefficients. 
\end{description}
\begin{figure}[h!]
	\centering
	\fbox{

		\begin{minipage}{13 cm}
\ARROWS{}: inputs - observed $y$ values, time horizon $n$, std deviation $\sigma$, $\delta \in (0,1]$, a hyper-parameter $\beta > 24$
\begin{enumerate}
    \item Initialize $t_h = 1$, $newBin=1$, $y_0 = 0$
    \item For $t$ = $1$ to $n$:
    \begin{enumerate}
        \item If $newBin == 1$, predict $x_t^{t_h} = y_{t-1}$, else predict $x_t^{t_h} = \bar{y}_{t_h:t-1}$ 
        \item set $newBin = 0$, observe $y_t$ and suffer loss $(x_t^{t_h} - \theta_t)^2$
        \item Let $\tilde{y} = pad_0(y_{t_h},...,y_{t})$  and $k$ be the padded length.
        \item Let $\hat{\alpha}(t_h:t) = T(H\tilde{y})$
        \item Restart Rule: If $\frac{1}{\sqrt{k}} \sum_{l=0}^{\log_2(k) - 1} 2^{l/2} \|\hat{\alpha}(t_h:t)[l] \|_1 > \frac{\sigma}{\sqrt{k}}$ then
        \begin{enumerate}
            \item set $newBin = 1$
            \item set $t_h = t+1$
        \end{enumerate}
    \end{enumerate}
\end{enumerate}
		\end{minipage}
	}
\end{figure}

Our policy is the byproduct of following question: How can one lift a batch estimator that is minimax over the TV class to a minimax online algorithm?

Restarting OGD when applied to our setting with squared error losses reduces to partitioning the duration of game into fixed size chunks and outputting online averages. As described in Section~\ref{gen_start}, this leads to suboptimal regret. However, the notion of averaging is still a good idea to keep. If within a time interval, the Total Variation (TV) is adequately small, then outputting sample averages is reasonable for minimizing the cumulative squared error. Once we encounter a bump in the variation, a good strategy is to restart the averaging procedure. Thus we need to adaptively detect intervals with low TV. For accomplishing this, we communicate with an oracle estimator whose output can be used to construct a lowerbound of TV within an interval. The decision to restart online averaging is based on the estimate of TV computed using this oracle. Such a decision rule introduces non-linearity and hence breaks free of the suboptimal world of linear forecasters.

The oracle estimator we consider here is a slightly modified version of the soft thresholding estimator from \citet{softThreshold95}. We capture the high level intuition behind steps (d) and (e) as follows. Computation of Haar coefficients involves smoothing adjacent regions of a signal and taking difference between them. So we can expect to construct a lowerbound of the total variation $\|D\theta_{1:n}\|_1$ from these coeffcients. The extra thresholding step $T(.)$ in (d) is done to denoise the Haar coefficients computed from noisy data. In step (e), a weighted L1 norm of denoised coefficients is used to lowerbound the total variation of the true signal. The multiplicative factors $2^{l/2}$ are introduced to make the lowerbound tighter. We restart online averaging once we detect a large enough variation. The first coefficient $\hat{\alpha}(t_h:t)_1$ is zero due to the re-centering caused by $pad_0$ operation. The hyper-parameter $\beta$ controls the degree to which we shrink the noisy wavelet coefficients. For sufficiently small $\beta$, It is almost equivalent to the universal soft-thresholding of \citep{softThreshold95}. The optimal selection of $\beta$ is described in Theorem~\ref{thm:main}.

We refer to the duration between two consecutive restarts inclusive of the first restart but exclusive of the second as a bin. The policy identifies several bins across time, whose width is adaptively chosen.

\begin{comment}

\begin{figure}[t]
    \centering
    \includegraphics[width=0.55\textwidth]{figs/spline_demo_TV.pdf}
    \includegraphics[width=0.285\textwidth]{figs/spline_err_TV.pdf}
    \caption{\emph{An illustration of \ARROWS{} on a sequence with heterogeneous smoothness. 
    We compare qualitatively (on the left) and quantitatively (on the right) to two popular baselines: (a) restarting online gradient descent \citep{besbes2015non}; (b) the moving averages \citep{box1970time} with optimal parameter choices. As we can see, \ARROWS{} achieves the optimal $\tilde O(n^{1/3})$ regret while the baselines are both suboptimal. }\yw{Please make the figure aligned. The white space is wasteful. Also, the Left pane seems to have clipped some part of the figure.\red{(is this ok?)}}}
    \label{fig:illustration}
\end{figure}
\end{comment}

\begin{figure}[htp]
  \centering
  \stackunder{\hspace*{-0.5cm}\includegraphics[width=0.6\textwidth,height=0.7\textheight,keepaspectratio]{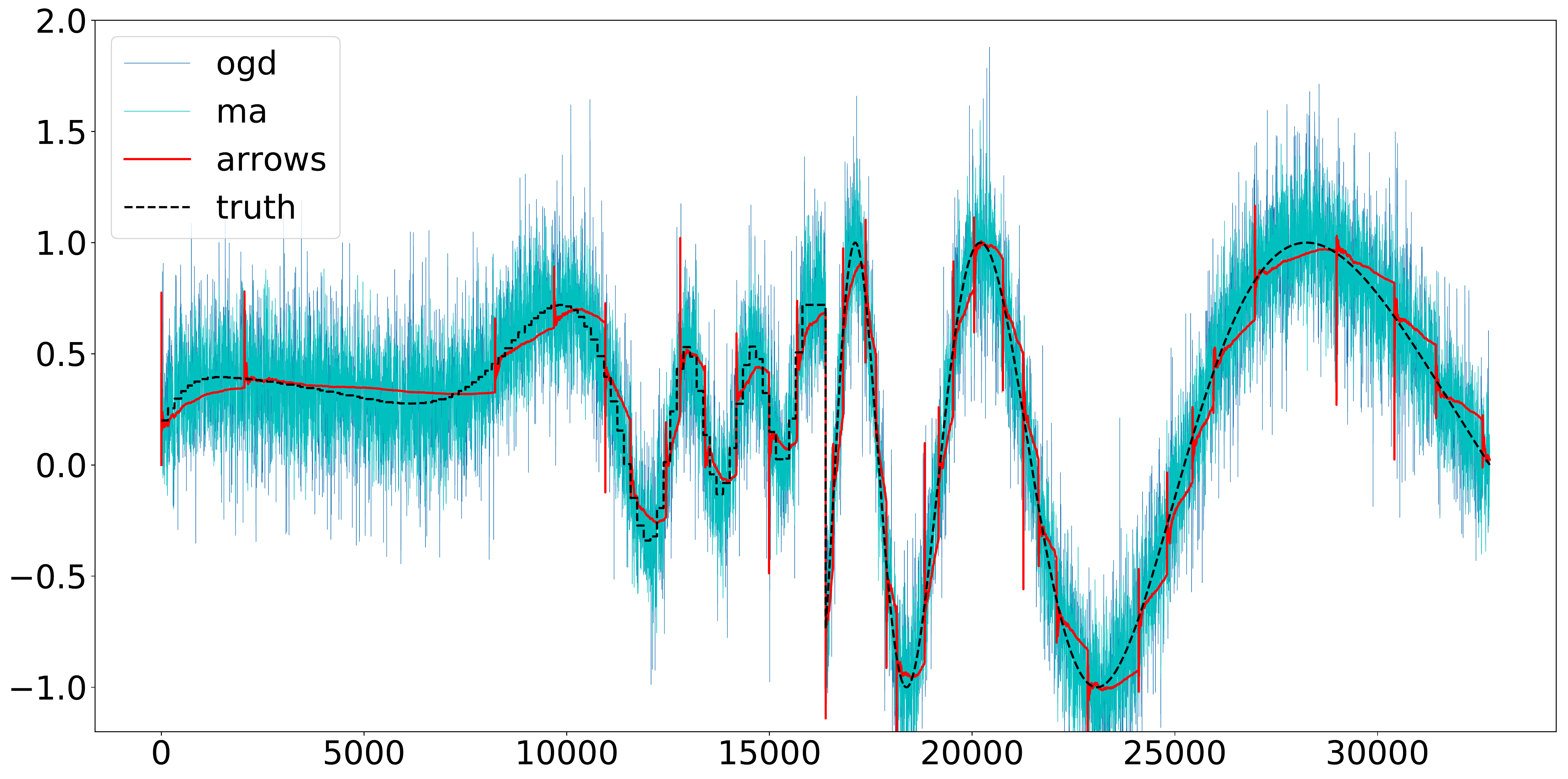}}{}%
  \stackunder{\includegraphics[width=0.4\textwidth,height=0.2\textheight]{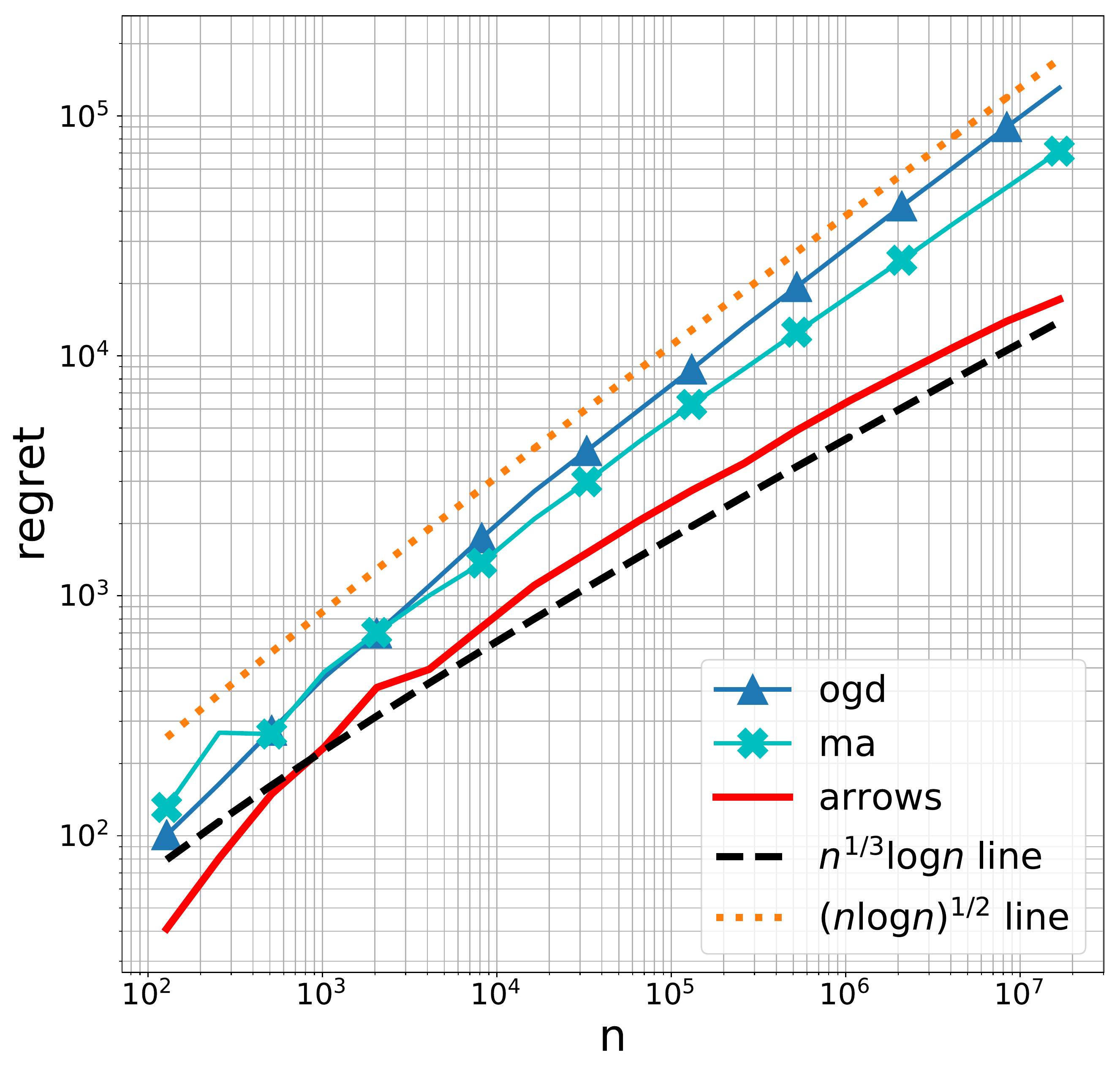}}{}
   \caption{\emph{An illustration of \ARROWS{} on a sequence with heterogeneous smoothness. 
    We compare qualitatively (on the left) and quantitatively (on the right) to two popular baselines: (a) restarting online gradient descent \citep{besbes2015non}; (b) the moving averages \citep{box1970time} with optimal parameter choices. As we can see, \ARROWS{} achieves the optimal $\tilde O(n^{1/3})$ regret while the baselines are both suboptimal. }}\label{fig:illustration}  
\end{figure}

\subsection{Dynamic Regret of \ARROWS{}}
\label{sec:analysis}
In this section, we provide bounds for non-stationary regret and run-time of the policy.

\begin{theorem}
    \label{thm:main}
    Let the feedback be $y_t = \theta_t + Z_t$, $t=1,\ldots,n$ and $Z_t$ be independent, $\sigma$-subgaussian random variables. 
    If $\beta = 24+\frac{8\log(8/\delta)}{\log(n)}$, then with probability at least $1-\delta$, \ARROWS{} achieves 
    a dynamic regret of $\tilde{O}(n^{1/3}\|D\theta_{1:n}\|_1^{2/3}\sigma^{4/3} +|\theta_1|^2 +  \|D\theta_{1:n}\|_2^2 + \sigma^2 )$ where $\tilde{O}$ hides a logarithmic factor in $n$ and $1/\delta$.
    %The dynamic regret of the policy is $\tilde{O}_p(n^{1/3}C_n^{2/3}\sigma^{4/3} + )$
\end{theorem}
\begin{proof}[Proof Sketch]
Our policy is similar in spirit to restarting OGD but with an adaptive restart schedule. The key idea we used is to reduce the dynamic regret of our policy in probability roughly to a sum of squared error of a soft thresholding estimator and number of restarts. This was accomplished by using a Follow The Leader (FTL) reduction. For bounding the squared error part of the sum we modified the threshold value for the estimator in \citet{softThreshold95} and proved high probability guarantees for the convergence of its empirical mean. To bound the number of times we restart, we first establish a connection between Haar coefficients and total variation. This is intuitive since computation of Haar coefficients can be viewed as smoothing the adjacent regions of a signal and taking their difference. Then we exploit a special condition called ``uniform shrinkage'' of the soft-thresholding estimator which helps to optimally bound the number of restarts with high probability.
\end{proof}

Theorem~\ref{thm:main} provides an upper bound of the minimax dynamic regret for forecasting the TV class.
\begin{corollary}
Suppose the ground truth $\theta_{1:n} \in TV(C_n)$ and $|\theta_1| \le U$. Then $\|D\theta_{1:n}\|_1 \le C_n$. By noting that $\|D\theta_{1:n}\|_2 \le \|D\theta_{1:n}\|_1$, under the setup in Theorem \ref{thm:main} \ARROWS{} achieves a dynamic regret of $\tilde{O}(n^{1/3}C_n^{2/3}\sigma^{4/3} +U^2 +  C_n^2 + \sigma^2 )$ with probability at-least $1-\delta$.
\end{corollary}

\begin{remark}[Adaptivity to unknown parameters.]
	Observe that \ARROWS{} does not require the knowledge of $C_n$.It adapts optimally to the unknown TV radius $C_n := \|D\theta_{1:n}\|_1$ of the ground truth $\theta_{1:n}$. The adaptivity to $n$ can be achieved by a standard doubling trick. $\sigma$, if unknown, can be robustly estimated from the first few observations by a Median Absolute Deviation estimator (eg. Section 7.5 of \citet{DJBook}), thanks to the sparsity of wavelet coefficients of TV bounded functions. 
\end{remark}

\begin{comment}
\begin{remark}
    As shown in Appendix~\ref{app:sobolev}, a rewriting of the regret bound reveals that our policy is also optimal for predicting sequences from Sobolev space defined by sequences that satisfy $\|D\theta_{1:n}\|_2 \le C'_n = n^{-1/2} C_n$. In other words, our policy is adaptively minimax, adaptive to the underlying function class being Sobolev or TV class.
\end{remark}
\end{comment}
\subsection{A lower bound on the minimax regret}\label{sec:precise_lowerbound}
We now give a matching lower bound of the expected regret, which establishes that \ARROWS{} is adaptively minimax.
\begin{proposition}\label{prop:lowerbound}
	Assume $\min\{U, C_n\} > 2\pi\sigma$ and $n > 3$, there is a universal constant $c$ such that
	$$
	\inf_{x_{1:n}} \sup_{\theta_{1:n}\in \TV(C_n)}  \E\left[\sum_{t=1}^n \big(x_t(y_{1:t-1}) - \theta_t\big)^2 \right] \geq  c(U^2 + C_n^2 + \sigma^2\log n +  n^{1/3} C_n^{2/3}\sigma^{4/3}).
	$$	
%	 no algorithm can achieve a regret smaller than $c(U^2 + C_n^2 + \sigma^2\log n +  n^{1/3} C_n^{2/3}\sigma^{4/3})$
%	with probability larger than $0.5$. 
\end{proposition}
The proof is deferred to the Appendix~\ref{app:lower}. The result shows that our result in Theorem~\ref{thm:main} is optimal up to a logarithmic term in $n$ and $1/\delta$ for almost all regimes (modulo trivial cases of extremely small $\min\{U,C_n\}/\sigma$ and $n$)\footnote{When both $U$ and $C_n$ are moderately small relative to $\sigma$, the lower bound will depend on $\sigma$ a little differently because the estimation error goes to $0$ faster than $1/\sqrt{n}$. We know the minimax risk exactly for that case as well but it is somewhat messy \citep[see e.g.][]{wasserman2006book}. When they are both much smaller than $\sigma$, e.g., when $ \min\{U,C_n\} \leq \sigma /\sqrt{n}$, then outputting $0$ when we do not have enough information will be better than doing online averages. }.

\begin{remark}[The price of forecasting]
	%The lower bound also implies that 
	The result also shows that \emph{forecasting is strictly harder than smoothing}.
	Observe that a term with $C_n^2$ is required even if $\sigma = 0$, whereas in the case of a one-step look-ahead oracle (or the smoothing algorithm that sees all $n$ observations) does not have this term. This implies that the total amount of variation that \emph{any} algorithm can handle while producing a sublinear regret has dropped from $C_n= o(n)$ to $C_n = o(\sqrt{n})$. Moreover, the regime where the $n^{1/3} C_n^{2/3}\sigma^{4/3}$ term is meaningful only when $C_n = o(n^{1/4})$. For the region where $\sigma n^{1/4} \ll C_n \ll \sigma n^{1/2}$, the minimax regret is essentially proportional to $C_n^2$. This is illustrated in Figure~\ref{fig:minimax_rate}.
\end{remark}

\begin{figure}[tbh]
\centering
\includegraphics[width=0.7\textwidth]{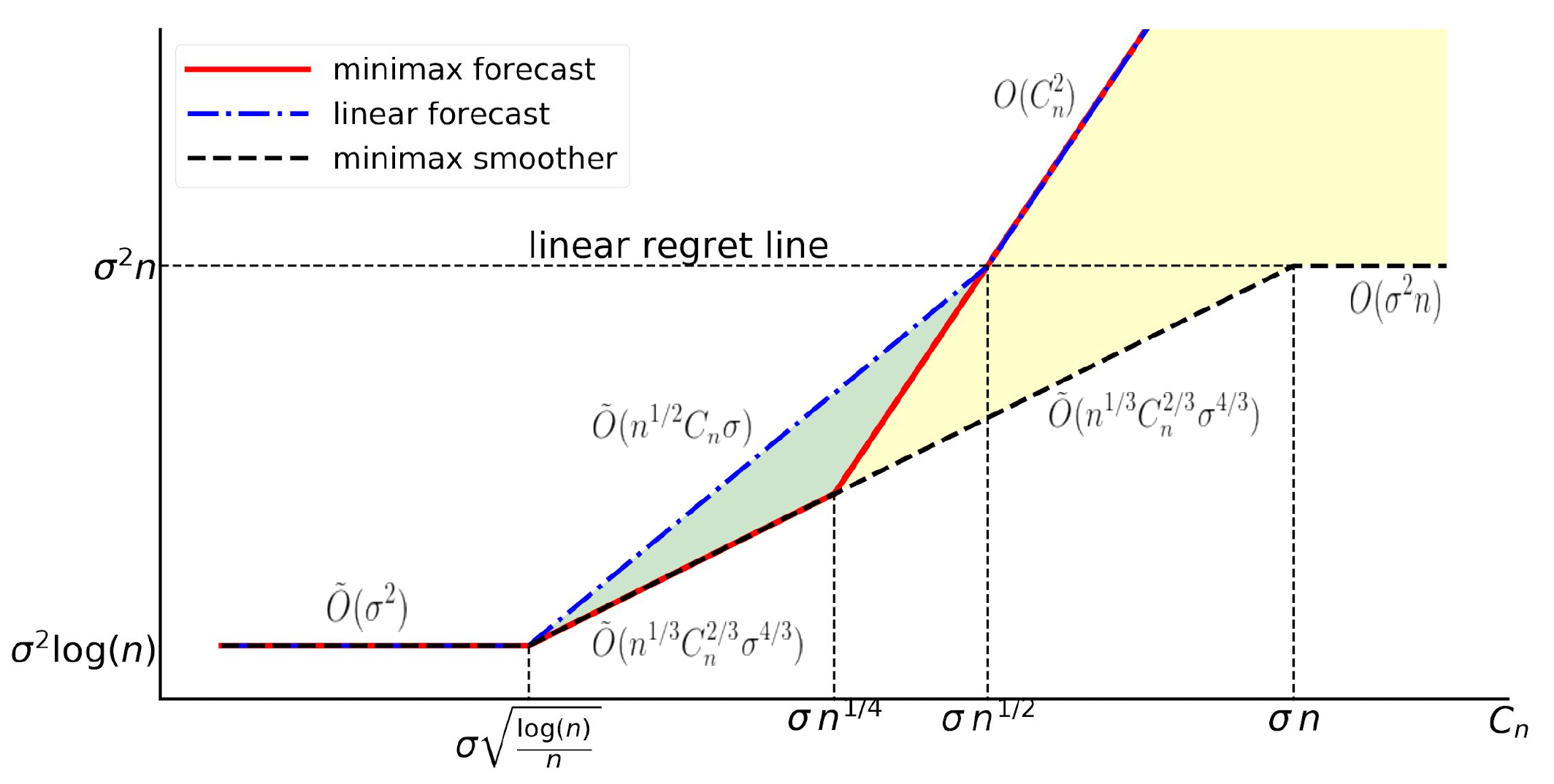}
\caption{\emph{An illustration of the minimax (dynamic) regret of forecasters and smoothers as a function of $C_n$. The non-trivial regime for forecasting is when $C_n$ lies between $\sigma\sqrt{\frac{\log(n)}{n}}$ and $\sigma\:n^{1/4}$ where forecasting is just as hard as smoothing. 
%In this regime, we study the prospects of designing forecasters whose dynamic regret matches the known rates from non-parametric estimation. 
When $C_n > \sigma\:n^{1/4}$, forecasting is harder than smoothing. The yellow region indicates the extra loss incurred by any minimax forecaster. The green region marks the extra loss incurred by a linear forecaster compared to minimax forecasting strategy. The figure demonstrates that linear forecasters are sub-optimal even in the non-trivial regime. When $C_n > \sigma\:n^{1/2}$, it is impossible to design a forecasting strategy with sub-linear regret. For $C_n > \sigma\:n$, identity function is optimal estimator for smoothing and when when $C_n < \sigma\sqrt{\frac{\log(n)}{n}}$, online averaging is optimal for both problems.}} \label{fig:minimax_rate}
\end{figure}

We note that in much of the online learning literature, it is conventional to consider a slightly more restrictive setting with bounded domain, which could reduce the minimax regret. The following remark summarizes a variant of our results in this setting.
\begin{remark}[Minimax regret in bounded domain]
    If we consider predicting sequences from a subset of the $TV(C_n)$ ball having an extra boundedness condition $|\theta_i| \le B$ for $i = 1 \ldots n$, it can be shown that \ARROWS{} achieves a dynamic regret of $\tilde{O}(n^{1/3}C_n^{2/3}\sigma^{4/3} +B^2 +  B C_n + \sigma^2 )$ with probability at least $1-\delta$. The discussion in Appendix \ref{app:lower}, establishes a matching lower bound in this more constrained setting. In particular, forecasting is still strictly harder than smoothing due to the $BC_n$ factor in the bound.%--- inst is the factor that separates the smoothing and forecasting in 
\end{remark}

\subsection{The adaptivity of \ARROWS{} to Sobolev and Holder classes}
It turns out that \ARROWS{} is also adaptively optimal in forecasting sequences in the discrete Sobolev classes and the discrete Holder classes, which are defined as
\begin{align}
    \cS(C_n') = \{\theta_{1:n} : \|D \theta_{1:n}\|_2 \le C_n' \}, &&     \cH(B_n') = \{\theta_{1:n} : \|D \theta_{1:n}\|_\infty \le B_n' \}.
\end{align}
These classes feature sequences that are more spatially homogeneous than those in the TV class. The minimax cumulative error of nonparametric estimation in the discrete Sobolev class is $\Theta(n^{2/3}[C'_n]^{2/3}\sigma^{4/3})$ \citep[see e.g.,][Theorem 5 and 6]{sadhanala2016total}.

\begin{corollary}
    \label{thm:main_sobelov1}
    Let the feedback be $y_t = \theta_t + Z_t$ where $Z_t$ is an independent, $\sigma$-subgaussian random variable. 
    Let $\theta_{1:n} \in \cS(C_n')$ and $|\theta_1| \le U$. If $\beta = 24+\frac{8\log(8/\delta)}{\log(n)}$, then with probability at least $1-\delta$, \ARROWS{} achieves 
    a dynamic regret of $\tilde{O}(n^{2/3}[C_n']^{2/3}\sigma^{4/3} +U^2 +  [C_n^']^2 + \sigma^2 )$ where $\tilde{O}$ hides a logarithmic factor in $n$ and $1/\delta$.
\end{corollary}

\begin{table}
	\centering
	\caption{\emph{Minimax rates for cumulative error $\sum_{i=1}^{n} (\hat{\theta}_i - \theta)^2$ in various settings and policies that achieve those rates.
			%The regret lowerbounds for forecasting setting are developed by invoking the corresponding lowerbounds from \citep{donoho1998minimax} in the smoothing setting (see Appendix \ref{app:lower}). 
			\ARROWS{} is adaptively minimax across all of the described function classes while linear forecasters fail to perform optimally over the TV class. For simplicity, we assume $U$ is small and hide a $\log n$ factors in all the forecasting rates.}
		%The contributions from this paper are highlighted in \blue{blue} color.
	}\label{tab:results_summary}
	\resizebox{\textwidth}{!}{
		\begin{tabular}{|cc|c|c|c|} 
			\hline
			\multicolumn{2}{|c|}{Class}                                     & \begin{tabular}[c]{@{}l@{}} Minimax rate for \\Forecasting\end{tabular} & \begin{tabular}[c]{@{}l@{}}Minimax rate for \\Smoothing\end{tabular} & \begin{tabular}[c]{@{}l@{}}Minimax rate for \\Linear Forecasting\end{tabular}  \\ 
			\hline
			TV      & $\|D\theta_{1:n}\|_1 \le C_n$  & $n^{1/3}C_n^{2/3}\sigma^{4/3} + C_n^2 + \sigma^2$     & $n^{1/3}C_n^{2/3}\sigma^{4/3} + \sigma^2$  & $n^{1/2}C_n\sigma + C_n^2 +              \sigma^2$                                                                              \\ 
			\hline
			Sobolev & $\|D\theta_{1:n}\|_2 \le C_n'$  & $n^{2/3}[C_n']^{2/3}\sigma^{4/3} + [C_n']^2 + \sigma^2$ & $n^{2/3}[C_n']^{2/3}\sigma^{4/3} + \sigma^2$ & $n^{2/3}[C_n']^{2/3}\sigma^{4/3} + [C_n']^2 + \sigma^2$                                                                            \\ 
			\hline
			Holder  & $\|D\theta_{1:n}\|_{\infty}\leq L_n$  & $nL_n^{2/3}\sigma^{4/3} + nL_n^2 + \sigma^2$  & $nL_n^{2/3}\sigma^{4/3} + \sigma^2$ & $nL_n^{2/3}\sigma^{4/3} + nL_n^2 + \sigma^2$                                                                              \\ 
			\clineB{1-5}{2.5}
			\multicolumn{2}{|c|}{Minimax Algorithm} & \ARROWS{} & \begin{tabular}[c]{@{}c@{}}Wavelet Smoothing\\ Trend Filtering \end{tabular} & \begin{tabular}[c]{@{}c@{}}Restarting OGD\\ Moving Averages \end{tabular} \\                  
			\clineB{1-5}{2.5}
		\end{tabular}
	}
	
	\smallskip
	
	\begin{minipage}{0.3\textwidth}
		\includegraphics[width=\textwidth]{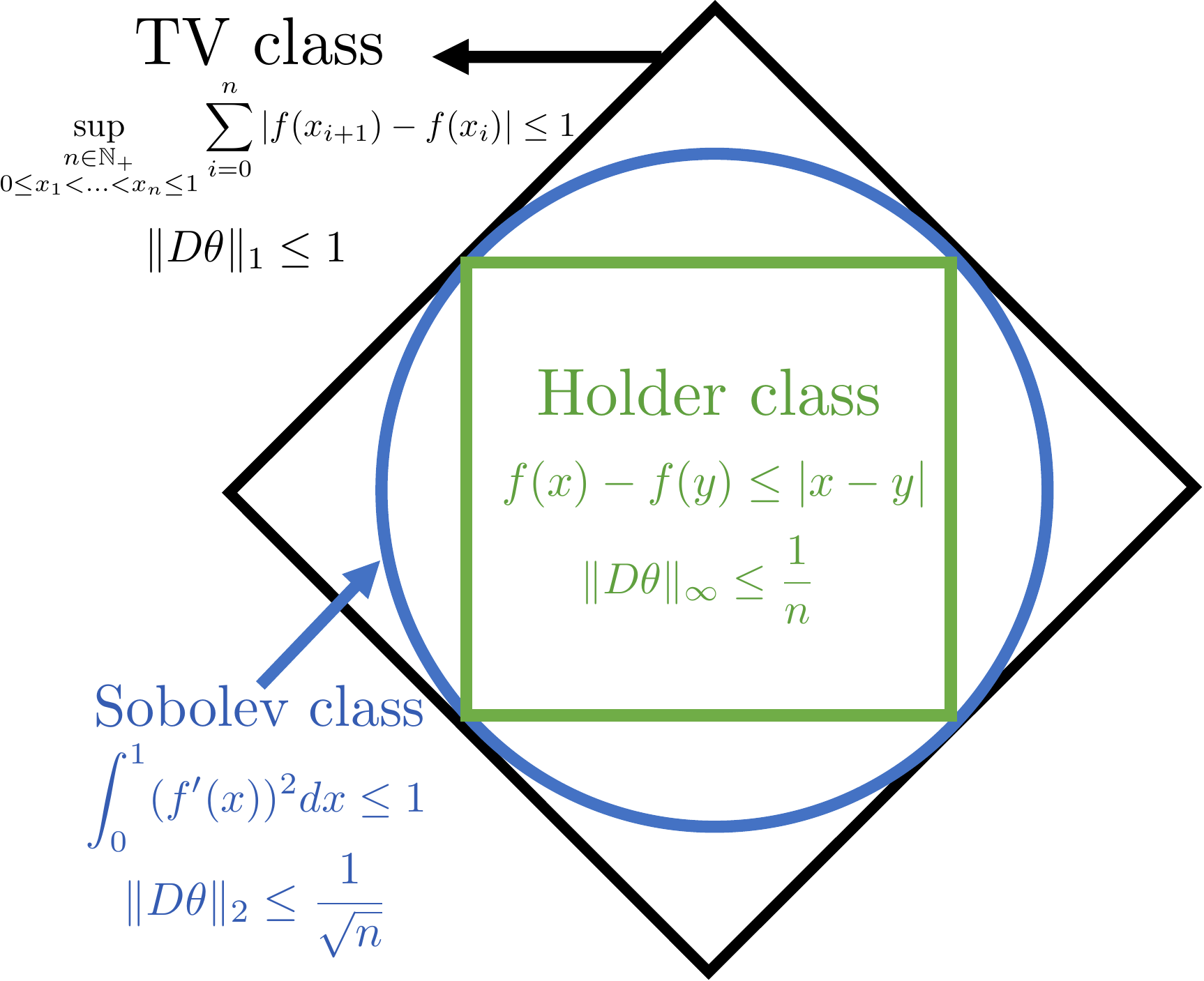}
	\end{minipage}
	\begin{minipage}{0.69\textwidth}
		\resizebox{\textwidth}{!}{
			\begin{tabular}{|cc|c|c|c|} 
				\hline
				\multicolumn{2}{|c|}{Canonical Scaling\footnote{The ``canonical scaling'' are obtained by discretizing functions in canonical function classes. Under the canonical scaling, Holder class $\subset$ Sobolev class $\subset$ TV class, as shown in the figure on the left.  \ARROWS{} is optimal for the Sobolev and Holder classes inscribed in the TV class. MA and Restarting OGD on the other hand require different parameters and prior knowledge of variational budget (i.e $C_n$ or $C'_n$) to achieve the minimax linear rates for the TV class and the Sobolev/Holder class.}}                                     & \begin{tabular}[c]{@{}l@{}}Forecasting\end{tabular} & \begin{tabular}[c]{@{}l@{}}Smoothing\end{tabular} & \begin{tabular}[c]{@{}l@{}}Linear Forecasting\end{tabular}  \\ 
				\hline
				TV      & $C_n \asymp 1$  & $n^{1/3}$ & $n^{1/3}$  & $n^{1/2}$                                                                              \\ 
				\hline
				Sobolev & $C_n' \asymp 1/\sqrt{n}$  & $n^{1/3}$ & $n^{1/3}$ & $n^{1/3}$                                                                            \\ 
				\hline
				Holder  & $L_n \asymp 1/n$  & $n^{1/3}$  & $n^{1/3}$ & $n^{1/3}$\\
				\hline
			\end{tabular}
		}
	\end{minipage}
	
\end{table}

Thus despite the fact that \ARROWS{} is designed for total variation class, it adapts to the optimal rates of forecasting sequences that are spatially regular. To gain some intuition,  let's minimally expand the Sobolev ball to a TV ball of radius $C_n = \sqrt{n}C_n'$. The chosen scaling of $C_n$ activates the embedding $\cS(C'_n) \subset TV(C_n)$ (see the illustration in Table~\ref{tab:results_summary}) with both classes having same minimax rate in the batch setting. This implies that dynamic regret of \ARROWS{} is simultaneously minimax optimal over $\cS(C'_n)$ and  $TV(C_n)$ wrt the term containing $n$. It can be shown that \ARROWS{} is optimal wrt to the additive $[C_n^']^2, U^2, \sigma^2$ terms as well. Minimaxity in Sobolev class implies minimaxity in Holder class since it is known that a Holder ball is sandwiched between two Sobolev balls having the same minimax rate \citep[see e.g.,][]{ryan_nonpar_notes}. A proof of the Corollary and related experiments are presented in Appendix~\ref{app:sobolev} and ~\ref{app:sobexp}.

\subsection{Fast computation}

Last but not least, we remark that there is a fast implementation of \ARROWS{} that reduces the overall time-complexity for $n$ step from $O(n^2)$ to $O(n\log n)$.
\begin{proposition}
	The run time of \ARROWS{} is $O(n\log(n))$, where $n$ is the time horizon.
\end{proposition}
The proof exploits the sequential structure of our policy and sparsity in wavelet transforms, which allows us to have $O(\log n)$ incremental updates in all but $O(\log n)$ steps.
See Appendix~\ref{sec:runtime} for details.

\subsection{Experimental Results}\label{demo}
To empirically validate our results, we conducted a number of numerical simulations that compares the regret of \ARROWS{}, (Restarting) OGD and MA. 
Figure~\ref{fig:illustration} shows the results on a function with heterogeneous smoothness (see the exact details and more experiments in Appendix~\ref{app:more_exp}) with the hyperparameters selected according to their theoretical optimal choice for the TV class (See Theorem~\ref{thm:ogd}, ~\ref{thm:ma} for OGD and MA in Appendix~\ref{gentle}).
The left panel illustrates that \ARROWS{} is locally  adaptive to heterogeneous smoothness of the ground truth. Red peaks in the figure signifies restarts. During the initial and final duration, the signal varies smoothly and \ARROWS{} chooses a larger window size for online averaging. In the middle, signal varies rather abruptly. Consequently \ARROWS{} chooses a smaller window size. On the other hand, the linear smoothers OGD and MA use a constant width and cannot adapt to the different regions of the space. This differences are also reflected in the quantitative evaluation on the right, which clearly shows that OGD and MA has a suboptimal $\tilde{O}(\sqrt{n})$ regret while \ARROWS{} attains the $\tilde{O}(n^{1/3})$ minimax regret!

\section{Concluding Discussion}
In this paper, we studied the problem of online nonparametric forecasting of bounded variation sequences. We proposed a new forecasting policy \ARROWS{} and proved that it achieves a cumulative square error (or dynamic regret) of %$\tilde{O}(n^{1/3}\|D\mathbf{\theta}\|_1^{2/3}\sigma^{4/3} + \sigma^2 + |\theta_1|^2 + \|D\mathbf{\theta}\|_2^2^2)$ 
$\tilde{O}(n^{1/3}C_n^{2/3}\sigma^{4/3} + \sigma^2 + U^2 + C_n^2)$
with total runtime of $O(n \log n)$. We also derived a lower bound for forecasting sequences with bounded total variation which matches the upper bound up to a logarithmic term 
which certifies the optimality of \ARROWS{} in all parameters. Through connection to linear estimation theory, we assert that no linear forecaster can achieve the optimal rate. 
\ARROWS{} is highly adaptive and has essentially no tuning parameters. We show that it is adaptively minimax (up to a logarithmic factor) simultaneously for all discrete TV classes, Sobolev classes and Holder classes with unknown radius. Future directions include generalizing to higher order TV class and other convex loss functions.

% which we show to enjoy a dynamic regret of $\tilde{O}(n^{1/3}C_n^{2/3}\sigma^{4/3} + \sigma^2 + U^2 + C_n^2)$ with runtime $O(n \log n)$.  We also derived a lowerbound which matches the upper bound up to a logarithmic term which certifies the optimality of \ARROWS{} in all parameters. Through connection to linear estimation theory, we assert that no linear forecaster can achieve the optimal rate. Adapting to $\sigma$ and generalizing to higher order TV class and other convex loss functions are considered as future directions to pursue.

\subsection*{Acknowledgement}
DB and YW were supported by a start-up grant from UCSB CS department and a gift from Amazon Web Services. The authors thank Yining Wang for a preliminary discussion that inspires the work, and Akshay Krishnamurthy and Ryan Tibshirani for helpful comments to an earlier version of the paper.

\begin{comment}
Currently our algorithm requires knowledge of parameters like time horizon, $C_n$ and $\sigma$. It would be interesting to explore the feasibility of designing a fully adaptive forecaster which requires no prior knowledge of these parameters and still perform optimally even in higher order total variation class. Towards this direction, we plan to study the implications of using trend filtering \citep{tibshirani2014adaptive} as the adaptive oracle. 
Finally, it will be interesting to explore to what extent we can generalize our argument to cover more general loss function. 
\end{comment}

\bibliography{tf,yuxiang}
\bibliographystyle{plainnat}
\appendix

\section{Discussion on other related works}\label{app:related}
\paragraph{Regret from Adaptive Optimistic Mirror Descent.}
In \citet{jadbabaie2015online}, the authors propose  Adaptive Optimistic Mirror Descent (AOMD) algorithm that minimizes the dynamic regret against a comparator sequence $\{u_t\}_{t=1}^{n}$. Their learning framework is the full information setting where learner predict $x_t \in \cX$ for a convex set $\cX \subseteq R^d$. Then a loss function $f_t(x)$ is revealed to the learner. To capture the regularity of the comparator, they define a quantity 
$C_n(u_1,u_2,...,u_n) := \sum_{t=1}^{n}\|u_t - u_{t-1}\|.$
They capture the regularity of loss functions by incorporating some external knowledge about their gradients via a predictable sequence $\{M_t\}_{t=1}^{n}$. They define:
$D_n := \sum_{t=1}^{n}\|\nabla f_t(x_t) - M_t \|_{*}^2$.
Finally to account for the temporal variability of $f_t$, they introduce $V_n$ as discussed earlier. The final regret bound is expressed in terms of these three quantities. However, their algorithm is adaptive and requires no prior knowledge about them. 

We note that our problem can be reduced to their framework if one considers loss functions $f_t(x) = (x - y_t)^2$.  Then the expected dynamic regret against the comparator sequence $\{\theta_t\}_{t=1}^{n}$ is given by
\begin{align}
    \sum_{t=1}^{n}E[(x-y_t)^2 - (\theta_t - y_t)^2]
    &= E[\sum_{t=1}^{n}(x-\theta_t)^2] \label{eqn:expected_reg},
\end{align}
where the expectation at right hand side is over the randomness of forecasting strategy. Hence we observe that their regret bound can be directly applied to bound the dynamic regret of our problem. It can be shown that (see Appendix~\ref{sec:regret_aomd}) given a fixed total variation bound $C_n = O(1)$, then $V_n$ and $D_n$ can be proved to be $O(n)$ with high probability. Plugging this into their regret bound yields an $\tilde{O}(\sqrt{n})$ rate in probability. However, it is unclear that whether AOMD is fundamentally limited by this rate or is there a potential suboptimality in their analysis of regret on our particular problem.

\paragraph{Other Dynamic Regret minimizing policies.} \citep{yang2016tracking} defines a path variation budget that is equivalent to our $C_n$ to characterize the sequence of convex loss functions.  However, under the noisy gradient feedback structure, they use a version of restarting OGD to get $C^{1/2}n^{1/2}$ regret rate. This is very similar to the policy in \citep{besbes2015non}. Since OGD is a linear forecaster, it is sub-optimal for predicting bounded variation sequences under the squared error metric.

In \citep{koolen2015minimax}, they consider minimizing the dynamic regret wrt to a comparator class that obeys  $\|D\theta_{1:n}\|_2 \le C'_n$. This is basically the discrete Sobolev class. As shown in appendix \ref{proof_main_thm}, our policy is minimax for forecasting such sequences as well when the observed values are noisy versions of the ground truth. However it should be noted that \citep{koolen2015minimax} does not have this distributional assumption on the observations.

\citep{chen2018smoothed} considers the Smoothed Online Convex Optimization framework where they simultaneously minimize the hitting loss $f_t$ and a switching cost. They provide dynamic regret bounds on this composite cost in the setting that $f_t$ is known to the learner before making the prediction. If we consider $f_t(x) = (x-y_t)^2$, then the baseline they compare against reduces to the offline Trend Filtering (TF) estimate when $\sum_{i=2}^{n} |x_t - x_{t-1}| \le L = C_n$. Then Theorem 10 of \citep{chen2018smoothed} states that the cumulative composite cost incurred by their proposed policy differs from that of the TF estimate by a term that is $O(\sqrt{nC_n})$. However, this doesn't translate to a meaningful regret bound in our setting.

\citep{hall2013dynamical} proposes the Dynamic Mirror Descent (DMD) algorithm that make use of a family of dynamical models for making prediction at each time step. They achieve a dynamic regret bound of $O(\sqrt{n}(1+V_{\phi_t}(\boldsymbol{\theta}_T)))$ where the second term measures the quality of the dynamical models in predicting ground truth.

\paragraph{Comparison to Online Isotonic Regression.} \citep{kotlowski2016} considers the dynamic regret minimization,
\begin{align}
    \sum_{t=1}^{n} (x_t-y_t)^2 - \min_{(\theta_1,...,\theta_n)} \sum_{t=1}^{n} (\theta_t - y_t)^2,
\end{align}
where $y_t \le B$ is a label revealed by the environment, $x_t \le B$ is the prediction of the learner, and the comparator sequence should obey $0 \le \theta_1 \le ... \le \theta_n \le B$. Here $B$ is a fixed positive number. Note that their setting and our framework are mutually reducible to each other in terms of regret guarantees via \ref{eqn:expected_reg}. They propose a minimax policy that achieves a dynamic regret of $\tilde{O}(n^{1/3})$ which translates to an $\tilde{O}(n^{1/3})$ in probability in our setting under the isotonic ground truth restriction. 

We note that the isotonic comparator sequence belong to a TV class of variational budget $C_n = B$. By using an argument similar to that in appendix \ref{sec:regret_aomd} which involves converting to deterministic noise setting and conditioning on a high probability event, it can be shown that our policy is out of the box minimax with high probability in isotonic framework when observations are noisy versions of an isotonic sequence.

\paragraph{Comparison to Online Non-Parametric regression methods.}
We note that our problem falls into the more general framework of online non-parametric regression setting studied in \citep{online_nonpar2015}. We can reduce our dynamic regret minimization to their framework by using a similar argument as above through \eqref{eqn:expected_reg}. Since the bounded TV class is sandwiched between Besov spaces $B^{1}_{1,q}$ for the range $1 \le q \le \infty$, the discussion in section 5.8 of \citep{online_nonpar2015} establishes that minimax growth w.r.t $n$ as $O(n^{1/3})$ in the online setting for TV class. Thus our bounds, modulo logarithmic factor, matches with theirs though we give the precise dependence on $C_n$ and $\sigma$ as well. It is worthwhile to point out that while the bound in \citep{online_nonpar2015} is non-constructive, we achieve the same bound via an efficient algorithm.

\citep{pierre2015} proposes a minimax policy wrt to comaparator functions that are Holder smooth. In particular, for the Holder class $H_1$ that satisfy $|f(x) - f(y)| \le \lambda |x-y|$, their algorithm yields a regret of $\tilde{O}(n^{1/3})$. It is known (\citep{ryan_nonpar_notes}) that $H_1$ is sandwiched between two Sobolev balls having the same minimax rate in the iid batch statistical learning setting. Since our policy is optimal for Sobolev space (appendix \ref{app:sobolev}), it is also optimal over Holder ball $H_1$ when the observations are noisy versions of a Holder smooth functions. Though the framework of \citep{pierre2015} doesn't impose this distributional assumption. The runtime of their policy for $H_1$ class is $O(n^{7/3}\log n)$. It should be noted that Sobolev and Holder classes are arguably easier to tackle than the TV class since both of them can be embedded inside a TV class.

\paragraph{Strongly Adaptive Regret.}
\citet{daniely2015strongly} introduced the notion of Strongly Adaptive (SA) regret where the online learner is guaranteed to have low static regret for any interval within the duration of the game. They also propose a meta algorithm which can convert an algorithm of good static regret to one with good SA regret. However low static regret for any interval doesn't help in our setting because in each interval we are competing with a stronger dynamic adversary. A notion of SA dynamic regret would an interesting topic to explore. 

For minimizing dynamic regret, \citet{zhang2018dynamic} proposed a meta policy that uses an algorithm  with good SA regret as its subroutine. Hence we can use their framework with squared error loss functions as discussed above. They show that OGD has an SA regret of $O(\log(n))$ for strongly convex loss functions. Using OGD as the subroutine and applying corollary 7 of their paper yields a bound $\tilde{O}(n)$. By a similar argument one gets the same linear regret rate when online newton step is used as the subroutine. However, we should note that their algorithm works without the knowledge of radius of the TV ball $C_n$.

\paragraph{Classical time series forecasting models.}
Finally, we note that our work is complementary to much of the classical work in time-series forecasting (e.g., Box-Jenkins method/ARIMA \cite{box1970time}, Hidden Markov Models \citep{baum1966statistical}). These methods aim at using dynamical systems to capture the recurrent patterns under a stationary stochastic process, while we focus on harnessing the nonstationarity. Our work is closer to the ``trend smoothing'' literature (e.g., the celebrated Hodrick-Prescott filter \citep{hodrick1997postwar}, trend filtering \citep{l1tf,tibshirani2014adaptive,hutter2016optimal}).

\section{Additional Experiments}\label{app:more_exp}
\begin{figure}[htp]
  \centering
  \stackunder{\hspace*{-0.5cm}\includegraphics[width=7.5cm,height=5.5cm]{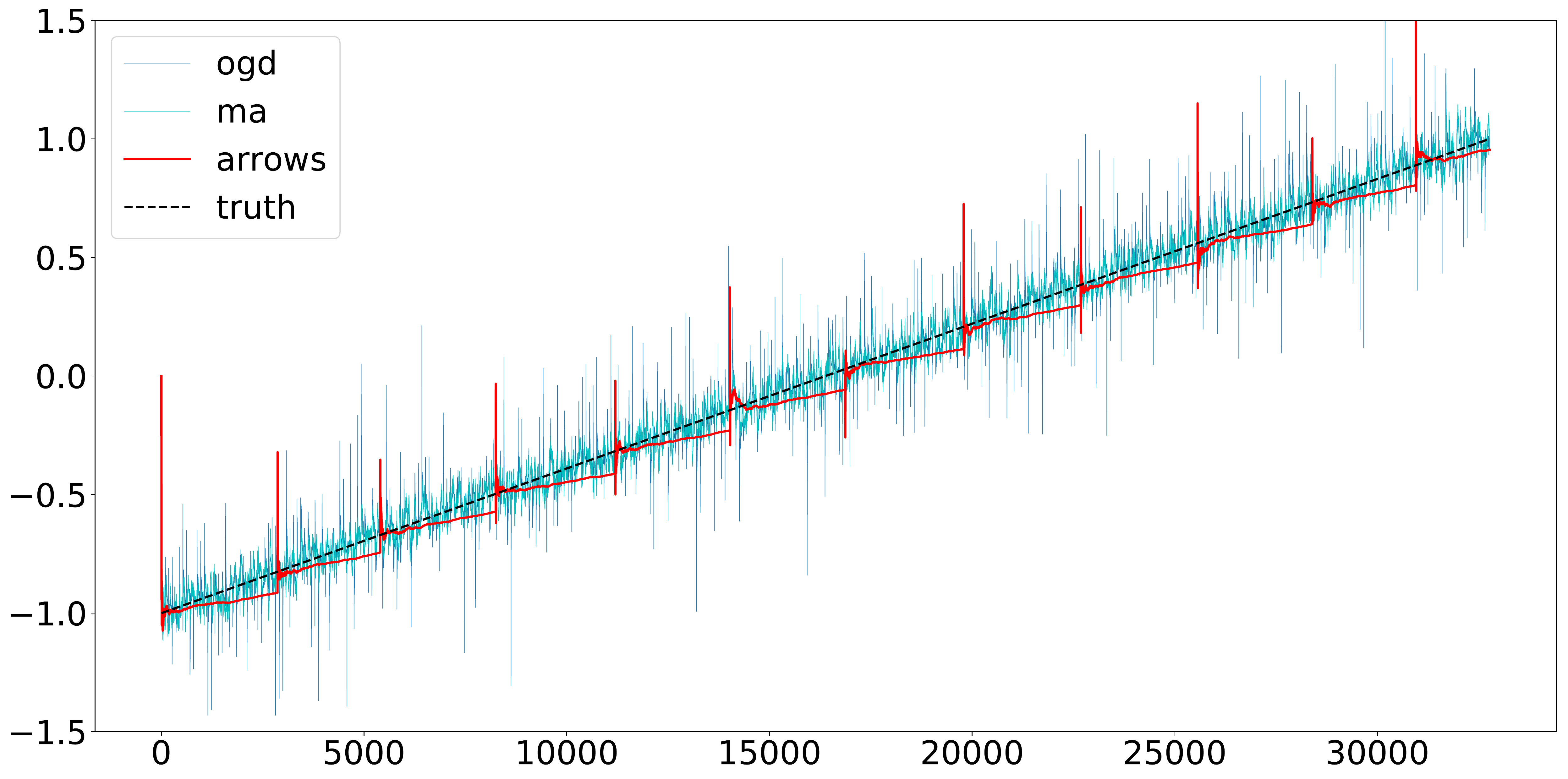}}{}%
  \stackunder{\includegraphics[width=6.9cm, height=5.5cm]{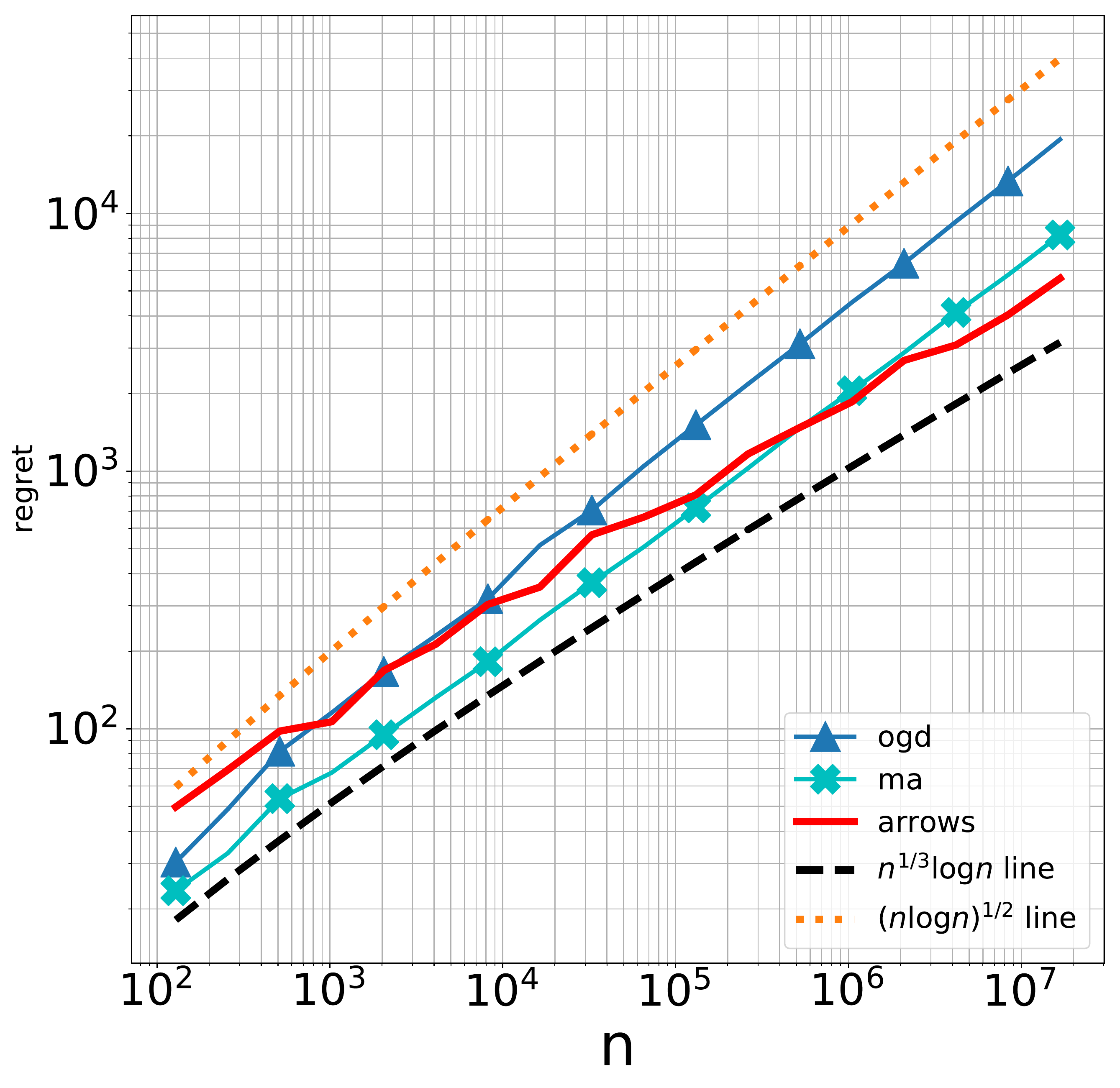}}{}
   \caption{\emph{An illustration of \ARROWS{} on a linear trend which has homogeneous smoothness}}\label{fig:linear}  
\end{figure}

\begin{figure}[htp]
  \centering
  \stackunder{\hspace*{-0.5cm}\includegraphics[width=7.5cm,height=5.5cm]{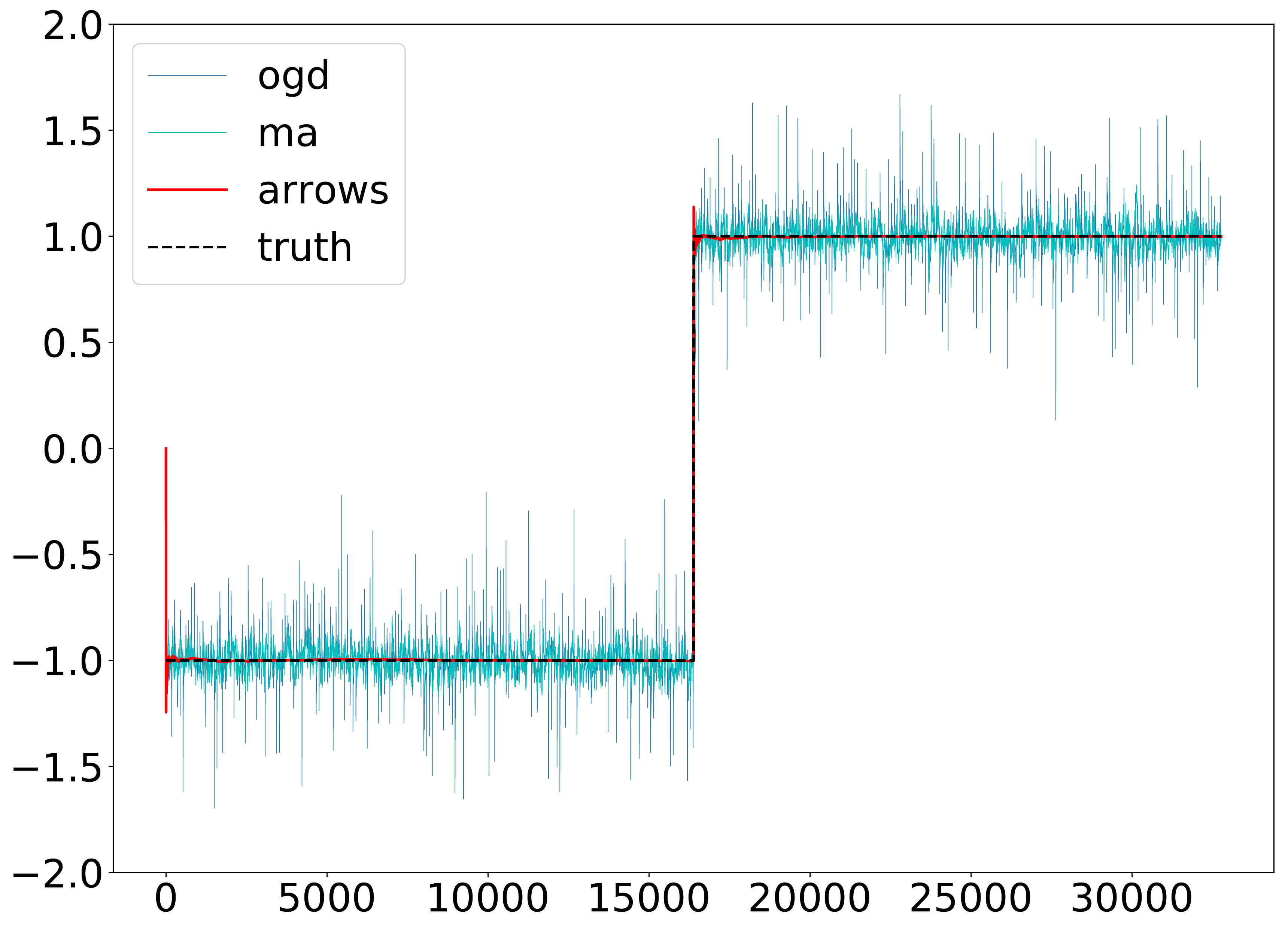}}{}%
  \stackunder{\includegraphics[width=6.9cm, height=5.5cm]{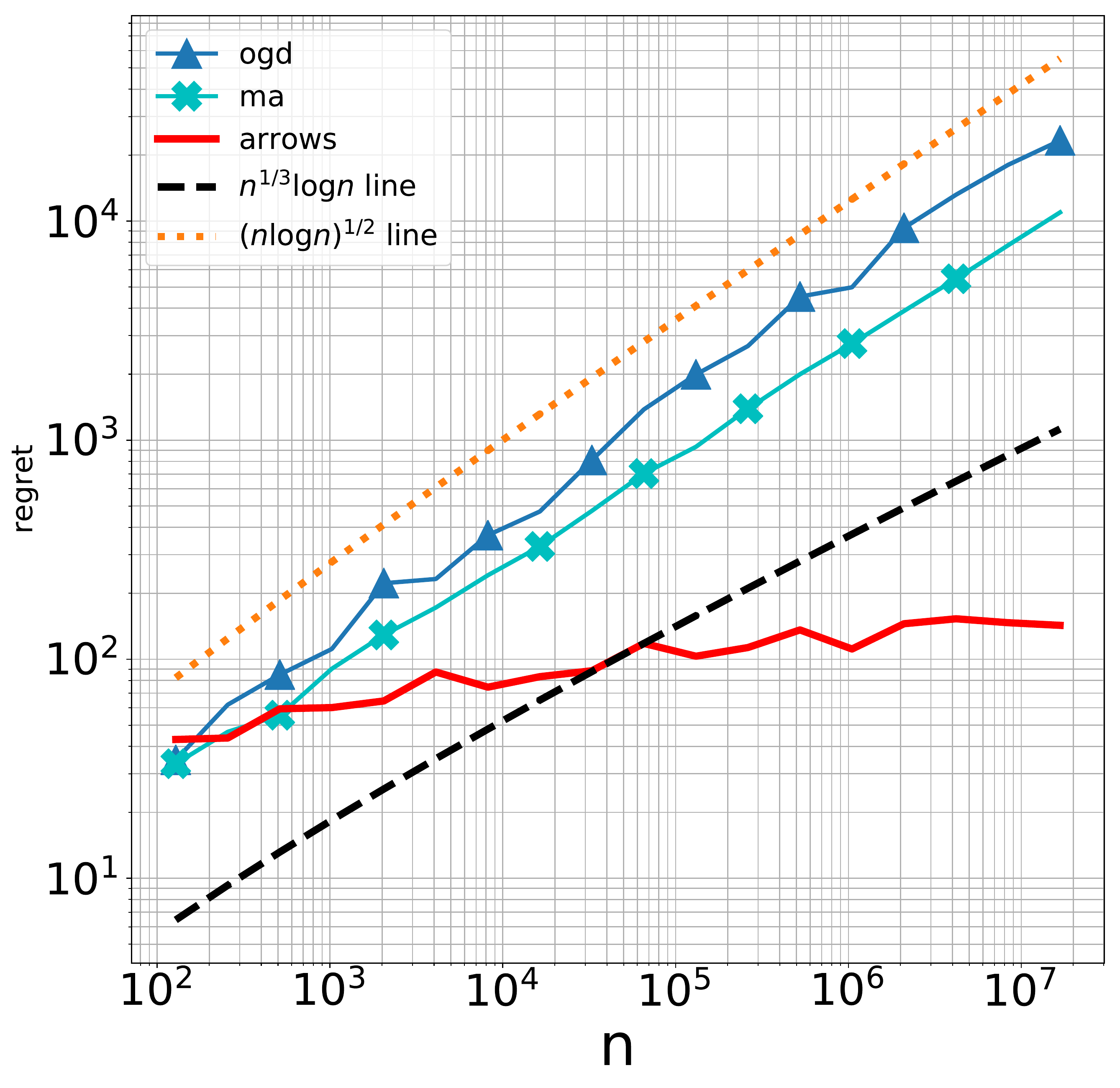}}{}
   \caption{\emph{An illustration of \ARROWS{} on a step trend with abrupt inhomogeneity.}}\label{fig:step}  
\end{figure}

The function that we generated in Figure~\ref{fig:illustration} is a hybrid function which in the first half is a ``discretized cubic spline'' with more knots closely placed towards the end. In the second half it is a Doppler function $f(t) = \sin \left(\frac{2\pi(1+\epsilon)}{t/n+ 0.38}\right)$ with $n$ being the time horizon. We observe noisy data $y_i = f(i/n) + z_i$, $i=1,...,n$ and $z_i$ are iid normal variables with $\sigma = 1$. The value of $C_n$ for $n > 60K$ is around 17. Hence for all $n>83521$, we are under the $n^{1/3}$ regime of $\sigma \sqrt{\log(n)/n} < C_n < \sigma n^{1/4}$.

The window size for moving averages and partition width of OGD were tuned optimally for the TV class (see Appendix C for details). Figure \ref{fig:illustration} depicts the estimated signals and dynamic regret averaged across 5 runs in a log log plot. The left panel illustrates that \ARROWS{} is locally  adaptive to heterogeneous smoothness of the ground truth. Red peaks in the figure signifies restarts. During the initial and final duration, the signal varies smoothly and \ARROWS{} chooses a larger window size for online averaging. In the middle, signal varies rather abruptly. Consequently \ARROWS{} chooses a smaller window size. On the other hand, the linear smoothers OGD and MA attains a suboptimal $\tilde{O}(\sqrt{n})$ regret.

%Plots for various other ground truth sequences is depicted in the Appendix B.

In Figure \ref{fig:linear} and \ref{fig:step} we plot the estimates and log-log regret for two more functions: A linear function that is homogeneously smooth and less challenging and a step function which has an abrupt discontinuity making it more inhomogeneous than linear but have lesser inhomogeneity w.r.t hybrid signal considered in \ref{demo}. Both OGD and MA were optimally tuned for the TV class as in Appendix \ref{gentle}.

The red peaks corresponds to restarts by \ARROWS{}. For linear functions we can see that ARROWS chooses inter-restart duration/bin-widths that are constant throughout. This is expected as a linear trend is spatially homogeneous. For the step function, we see that \ARROWS{} restart only once since the start. Further, notice that it quickly restarts once the bump is hit. For both of these functions, necessary scaling is done so that we are in the $n^{1/3}$ regime quite early.

\section{Upper bounds of linear forecasters}\label{gentle}
In this section we compute the optimal batch size for Restarting OGD and optimal window size for moving averages to yield the $\tilde{O}(\sqrt{n})$ regret rate.

\begin{theorem} \label{thm:ogd}
    Let the feedback be $y_t = \theta_t + Z_t$ where $Z_t$ is an independent, $\sigma$-subgaussian random variable. 
    Let $\theta_{1:n} \in \TV(C_n)$. Restarting OGD with batch size of $\sqrt{n \log n}\frac{\sigma}{C_n}$ achieves an expected dynamic regret of $\tilde{O}(U^2 + C_n^2 + \sigma C_n \sqrt{n})$.
\end{theorem}
\begin{proof}
Note that in our setting with squared error losses $f_t(x) = (x - \theta_t)^2$, the update rule of restarting OGD reduces to computing online averages. Thus OGD essentially divides the time horizon $n$ into fixed size batches and output online averages within each batch. Our objective here is to compute the optimal batch size that minimizes the dynamic regret.

We will bound the expected regret. Let $x_t$  be the estimate of OGD at time $t$. Let batches be numbered as $1,...,\ceil{n/L}$ where $L$ is the fixed batch size. Let the total variation of ground truth within batch $i$ be $C_i$. Time interval of batch $i$ is denoted by $[t_h^{(i)}, t_l^{(i)}]$. Due to bias variance decomposition within a batch we have,
\begin{align}
    R_i = \sum_{t=t_h^{(i)}}^{t_l^{(i)}} E[(x_t-\theta_t)^2]
    &= (\theta_{t_h^{(i)}-1} - \theta_{t_h^{(i)}})^2 + \sum_{t=t_h^{(i)}+1}^{t_l^{(i)}} (\theta_t - \bar{\theta}_{t_h^{(i)}:t-1})^2 + \frac{\sigma^2}{t-t_h^{(i)}} \label{eqn:ogd_decomp},\\
    &\le (\theta_{t_h^{(i)}-1} - \theta_{t_h^{(i)}})^2 + L C_i^2 + \sigma^2(2+\log L),
\end{align}
with the convention $\theta_0 = 0$ and at start of bin our prediction is just the noisy realization of the previous data point.

Summing across all bins gives,
\begin{align}
    \sum_{i=1}^{\ceil{n/L}} R_i
    &\le LC_n^2 + 2\sigma^2\frac{n (2+\log L)}{L} + U^2 + C_n^2.
\end{align}
where we have used assumption (A4) to bound the bias of the first prediction. The above expression can be minimized by setting $L = \sqrt{n \log n}\frac{\sigma}{C_n}$ to yield a regret bound of $O(U^2 + C^2 + \sigma C_n \sqrt{n \log n})$
\end{proof}

\begin{theorem}\label{thm:ma}
    Under the same setup as in Theorem \ref{thm:ogd}, moving averages with window size $\frac{\sigma \sqrt{n}}{C_n}$ yields a dynamic regret of $O(\sigma C_n \sqrt{n} + U^2 + C_n^2)$
\end{theorem}
\begin{proof}
Let the window size of moving averages be denoted by $m$. Consider the prediction at a time $x_t, t \ge m$. By bias variance decomposition we have,
\begin{align}
    E[(x_t - \theta_t)^2]
    &= \left(\theta_i - \frac{\sum_{j=i-m}^{i-1} \theta_j}{m} \right)^2 + \frac{\sigma^2}{m}.
\end{align}
By Jensen's inequality,
\begin{align}
    \left(\theta_i - \frac{\sum_{j=i-m}^{i-1} \theta_j}{m} \right)^2
    &\le \frac{\sum_{j=i-m}^{i-1} (\theta_j - \theta_i)^2}{m},\\
    &\le \frac{2\sum_{j=i-m}^{i-1}(j-i+1+m)(\theta_{j+1} - \theta_j)^2 }{m} \text{ ,by $(a+b)^2 \le 2a^2 + 2b^2$}.
\end{align}
Notice that the term $(\theta_i - \theta_{i-1})^2$ will be multiplied by a factor $m$ in the above bias bound at time point $i$, $m-1$ times in the next time point $i+1$ and so on. By summing this bias bound across the times points, we obtain
\begin{align}
   \sum_{i=m}^{n} \frac{2\sum_{j=i-m}^{i-1}(j-i+1+m)(\theta_{j+1} - \theta_j)^2 }{m}
   &\le 4m \sum_{i=1}^{n-1} (\theta_i - \theta_{i+1})^2 + U^2,\\
   &\le 4m C_n^2 + U^2.
\end{align}

The squared bias for the initial points can be bounded by.
\begin{align}
    \sum_{i=1}^{m-1}(\theta_i - \hat{\theta}_{(1:i-1)})^2
    &\le U^2 + C_n^2.
\end{align}

Summing the variance terms yields,
\begin{align}
    \sum_{t=1}^{n}\Var{(x_t)}
    &= \sum_{t=1}^{m-1}\frac{\sigma^2}{t} + \sum_{t=m}^{n}\frac{\sigma^2}{m},\\
    &\le \frac{(1+\log m+n)\sigma^2}{m}.
\end{align}

Thus the total MSE can be minimized by setting $m = \frac{\sigma \sqrt{n}}{C_n}$, we obtain a dynamic regret bound of $O(\sigma C_n \sqrt{n} + U^2 + C_n^2)$

\end{proof}

\section{Proof of useful lemmas}

We begin by recording an observation that follows directly from the policy.
\begin{lemma}
    \label{observe}
    For $m^{th}$ bin that spans the interval $[t_{h}^{(m)},t_{l}^{(m)}]$, discovered by the policy, let the lengths of  $\hat{\alpha}(t_{h}^{(m)}:t_{l}^{(m)}-1)$ and $\hat{\alpha}(t_{h}^{(m)}:t_{l}^{(m)})$ be $k$ and $k^{+}$ respectively. Then\\
    $ \sum_{l=0}^{\log_2 (k) -1} 2^{l/2} \|\hat{\alpha}(t_{h}^{(m)}:t_{l}^{(m)}-1)[l]\|_1 \le \sigma$ and
    $\sum_{l=0}^{\log_2 (k^+) - 1} 2^{l/2} \|\hat{\alpha}(t_{h}^{(m)}:t_{l}^{(m)})[l]\|_1 > \sigma $
\end{lemma}

Next we prove the marginal sub-gaussianity of the wavelet coefficients.

\begin{lemma}
\label{lemm:subgauss}
    Consider the observation model $y_i = \theta_i + \sigma z_i$, where $z_i$ is iid sub-gaussian with parameter 1, $i=1,..,n$. Let $\alpha_i$ denote the wavelet coefficients of the sequence $z = pad_0(y_1,...y_n)$. Then each $\alpha_i$ is sub-gaussian with parameter $2\sigma$.
\end{lemma}
\begin{proof}
    Without loss of generaility let's charecterize $\alpha_1$. Let $\boldsymbol{u} = [u_1,...u_n]^T$ denote the first row of the orthonormal wavelet transform matrix. Then,
    \begin{align}
        \alpha_1 = \sum_{i=1}^{n}y_i\left(u_i(1-\frac{1}{n}) - \sum_{j=1,j \neq i}^{n} \frac{u_j}{n}\right).
    \end{align}
    Thus $\alpha_1$ is a differentiable function of iid sub-gaussian noise $z_i$. We can find its Lipschtiz constant by bounding the gradient w.r.t $z_i$ as follows,
    \begin{align}
        \|\nabla \alpha_1(z_1,...,z_n)\|_2 
        &\le \sigma \left( \sum_{i=1}^{n} 2u_i^2 (1-\frac{1}{n})^2  +\frac{2}{n} \sum_{j=1,j \neq i}^{n} u_j^2 \right )^{\frac{1}{2}},\\
        &\le \sigma \left( 2 + 2\right)^{\frac{1}{2}},\\
        &= 2\sigma.
    \end{align}
    By proposition 2.12 in \citet{DJBook} we conclude that $\alpha_1$ sub-gaussian with parameter $2\sigma$.
\end{proof}

In the next lemma, we record the uniform shrinkage property of soft-thresholding estimator.
\begin{lemma}
    \label{lemma:threshold}
    For any interval $[t_h,t_l]$, let $Y = pad_0(y_{t_h},...,y_{t_l})$ and $\Theta = pad_0(\theta_{t_h},...,\theta_{t_l})$. Then $|(T(HY))_i| \le |(H\Theta)_i|$ with probability at-least $1-2n^{3-\beta/8}$ for each co-ordinate $i$.
    \end{lemma}
\begin{proof}
Consider a fixed bin $[\ubar{l},\bar{l}]$ with zero padded vector $Y \in R^k$. Due to sub-gaussian tail inequality, we have $|(HY)_i - (H\Theta)_i| \le \sigma\sqrt{\beta\log(n)}$ with probability at-least $1-2/n^{\beta/8}$. Consider the case $(H\Theta)_i \ge \sigma\sqrt{\beta\log(n)}$. Then both the scenarios $(HY)_i \le \sigma\sqrt{\beta\log(n)}$ and $(HY)_i > \sigma\sqrt{\beta\log(n)}$ leads to shrinkage to a value that is smaller than $|(H\Theta)_i|$ in magnitude due to soft-thresholding with threshold set to $\sigma\sqrt{\beta\log(n)}$. Now consider the case when $0 \le (H\Theta)_i \le \sigma\sqrt{\beta\log(n)}$. Again, soft-thresholding in both scenarios $(HY)_i \le \sigma\sqrt{\beta\log(n)}$ and $\sigma\sqrt{\beta\log(n)} \le (HY)_i \le (H\Theta)_i + \sigma\sqrt{\beta\log(n)}$ leads to shrinkage to a value that is smaller than $|(H\Theta)_i|$ in magnitude. One can come up with a similar argument for the case where $(H\Theta)_i \le 0$. Now applying a union bound across all $O(n)$ co-ordinates and all $O(n^2)$ bins, we get the statement of the lemma.
\end{proof}

\begin{lemma}
    \label{lemma:binbound}
    The number of bins, $M$, discovered by the policy is at-most $\max\{1,2n^{1/3}C_n^{2/3}\sigma^{-2/3}\log(n)\}$ with probability at-least $1-2n^{3-\beta/2}$.
\end{lemma}
\begin{proof}
Let $\Theta_m = [\theta_{1}^{(m)},\theta_{2}^{(m)},...,\theta_{p}^{(m)}]^T$ be the mean subtracted and zero padded ground truth sequence values in $m^{th}$ bin $[\ubar{l},\bar{l}]$ discovered by our policy. $y^{(m)} = [y_{1}^{(m)},y_{2}^{(m)},...,y_{p}^{(m)}]^T$ be the corresponding mean subtracted and zero padded observations. Note that due to zero padding $p \le 2(\bar{l}-\ubar{l})$ and some of the last values in the vector can be zeroes. Let $\alpha_m(\ubar{l}:\bar{l}) = H\Theta$ denotes the discrete wavelet coefficient vector. We can view the computation of the Haar coefficients as a recursion. At each level $l$ of the recursion, the entire length $p$, is divided into $2^l$ intervals. Let the sample averages of elements of $\Theta_m$ in these intervals be denoted by the sequence $\{\tilde{\theta}_1,\tilde{\theta}_2,...,\tilde{\theta}_{2^l}\}$. Let $\alpha_m^{(l)} \in \mathbb{R}^{2^{l}}$ denotes the vector of Haar coefficients at level $l$. 

First note that the Haar coefficient $\alpha_m^{(l)}(i) = \frac{1}{2} \sqrt{\frac{p}{2^l}} (\tilde \theta_{2i} - \tilde \theta_{2i-1})$ with $i=1,...,2^{l}$. 

\begin{align}
    \|\alpha_m^{(l)}\|_1^2
    &\le \frac{p}{2^{l+2}}\left( \sum_{i=1}^{2^{l}} | \tilde \theta_{2i} - \tilde \theta_{2i-1}| \right)^2,\\
    &\le \frac{p TV^2[\ubar{l}-1:\bar{l}]}{2^{l}},
\end{align}
where $TV[a,b]$ denotes the total variation of the true sequence in the interval $[a,b]$. The last inequality holds because the total variation of the smoothed sequence must be at-most four times the entire total variation of true sequence. The factor 4 is due to the fact that total variation when we pad a mean zero sequence with zeroes is at-most twice the total variation before zero padding.

We have,

\begin{align}
    \frac{1}{\sqrt{p}}\sum_{l=0}^{\log_2 (p) - 1} 2^{l/2} \|\alpha_m^{(l)}\|_1
    &\le \log p \: TV[\ubar{l}-1:\bar{l}].
\end{align}

In the policy we compute $\hat{\alpha}_m(\ubar{l}:\bar{l}) = T(Hy^{(m)})$ with the soft thresholding factor of $\sigma\sqrt{\beta\log(n)}$. From lemma \ref{lemma:threshold}, we have $|(T(Y))_i| \le |(H\Theta)_i| \: \forall i \in [1,p]$ with probability at-least $1-2n^{3-\beta/8}$. Since $[\ubar{l},\bar{l}]$ is a bin discovered by policy, lemma \ref{observe} gives a lowerbound on $\|\alpha_m(\ubar{l}:\bar{l})\|$ . Putting it all together yields the relation,
\begin{equation}
    \label{eq:sandwitch}
    \frac{\sigma}{\sqrt{p}} < \frac{1}{\sqrt{p}} \sum_{l=0}^{\log_2 (p) - 1} 2^{l/2} \|\hat \alpha_m^{(l)}(\ubar{l}:\bar{l})\|_1 \le  \frac{1}{\sqrt{p}} \sum_{l=0}^{\log_2 (p) - 1} 2^{l/2} \|\alpha_m^{(l)}(\ubar{l}:\bar{l})\|_1 \le \log(p) \:TV[\ubar{l}-1:\bar{l}],
\end{equation}

with probability at-least $1-2n^{3-\beta/8}$.

Thus the total variation in the time interval $[\ubar{l}-1,\bar{l}]$ can be lower bounded in probability as
\begin{equation}
    \label{eqn:tvlb}
    TV[\ubar{l}-1:\bar{l}] > \frac{\sigma}{\sqrt{p}\log n}.
\end{equation}
Due to assumption $(A3)$ we have,
\begin{equation}
    \label{eqn:tviden}
    \sum_{i=1}^{M}TV[\ubar{l}^{i}-1:\bar{l}^{i}] = C_n,
\end{equation}
where $[\ubar{l}^{i}:\bar{l}^{i}]$ are the bins discovered by the policy. 

Let $p_i$ be the padded width of bin $i$ discovered by the policy. Then,

\begin{align}
    C_n \log n
    &\ge \sum_{i=1}^{M} \frac{\sigma}{\sqrt{p_i}},\\
    &\ge \frac{M^2\sigma}{\sum_{i=1}^{M}\sqrt{p_i}},
\end{align}

where the last line is obtained via Jensen's inequality. Now using Holder's inequality $\|x\|_\beta \le d^{\frac{1}{\beta} - \frac{1}{\alpha}}\|x\|_\alpha$ for $0 < \beta \le \alpha$, $x \in \mathbb{R}^d$ with $\alpha = 1/2$, $\beta = 1$ and noting that $\sum_{i=1}^{M}p_i \le 2T$ gives,

\begin{align}
    \sigma M^2
    &\le C_n \log n \sum_{i=1}^{M}\sqrt{p_i},\\
    &\le C_n \log n \sqrt{Mn}.
\end{align}

Hence we get $M \le (2n)^{1/3}(C_n \log n)^{2/3} \sigma^{-2/3} \le 2n^{1/3}C_n^{2/3}\sigma^{-2/3}\log(n)$.

When $C_n = 0$, \eqref{eq:sandwitch} implies that our policy will not restart with probability at-least $1-2n^{3-\beta/8}$ making $M=1$.
\end{proof}
\begin{comment}
\begin{lemma}
For the $m^{th}$ bin $[t_{h}^{(m)},t_{l}^{(m)}]$ discovered by the policy,
\begin{equation}
    \label{eqn:ftl}
    \sum_{t=t_{h}^{(m)}}^{t_{l}^{(m)}}(\bar{\theta}_{t_{h}^{(m)}:t} - \theta_t)^2 \le (\theta_{t_h^{(m)}}-\theta_{t_h^{(m)}-1})^2 + 2TV^2[t_{h}^{(m)}-1:t_{l}^{(m)}]\log(t_{l}^{(m)}-t_{h}^{(m)}) + \sum_{t=t_{h}^{(m)}}^{t_{m}^{(m)}-1} (\theta_t - \bar{\theta}_{t_{h}^{(m)}:t_{l}^{(m)}})^2
\end{equation}
\end{lemma}
\end{comment}

We restate two useful results from \citet{softThreshold95}.
\begin{lemma}
    \label{lemma:dj1}
    Consider the observation model $y=\alpha + Z$, where $y \in R^k$ and $|Z_i| \le \delta \forall\:i \in [1,k]$. Let $\hat{\alpha}_\delta$ be the soft thresholding estimator with input $y$ and threshold $\delta$, then
    \begin{equation}
        \|\hat{\alpha}_\delta - \alpha \|^2 \le \sum_{i=1}^{k} min\{\alpha_i^2,4\delta^2\}.
    \end{equation}
\end{lemma}

\begin{lemma}
    \label{lemma:dj2}
    Consider the observation model $y=\alpha + Z$, where $y \in R^k$, $\alpha \in A$ and each $Z_i$ is sub-gaussian with parameter $\sigma^2$. If A is solid and orthosymmetric, then
    \begin{equation}
        \inf_{\hat{\alpha}}\sup_{\alpha\in A} E[\|\hat{\alpha}-\alpha\|^2] \ge \frac{1}{2.22}\sup_{A} \sum_{i=1}^{k} min\{\alpha_i^2,\sigma^2\}.
    \end{equation}
\end{lemma}
Let's pause a moment to ponder how remarkable the above lemma is. The observations need not be even iid. Given $A$ is solid and orthosymmetric, all that is required is the marginal sub-gaussianity as the soft-thresholding operation works co-ordinate wise. Now we reprove theorem 4.2 from \citet{softThreshold95} with a slight modification of threshold value in the estimator.
\begin{theorem}
\label{thm:dj}
With probability at-least $1-2n^{1-\beta/2}$, under the model in lemma \ref{lemma:dj2}, the soft thresholding estimator $\hat{\alpha}_\delta$ with $\delta = \sigma\sqrt{\beta\log(n)}$ obeys
\begin{equation}
    \label{eqn:minimax}
    \|\hat{\alpha}_\delta - \alpha \|^2 \le 80(1+\log(n)) \inf_{\hat{\alpha}}\sup_{\alpha\in A} E[\|\hat{\alpha}-\alpha\|^2].
\end{equation}
\end{theorem}
\begin{proof}
Consider the soft thresholding estimator $\hat{\alpha}_\delta$. By Gaussian tail inequality and union bound we have $P(\sup_{i}|Z_i| \ge \delta) \le 2n^{1-\beta/2}$. Conditioning on the event $\sup_{i}|Z_i| \le \delta$ and applying lemma \ref{lemma:dj1},
\begin{align}
    \|\hat{\alpha}_\delta - \alpha \|^2 
    &\le \sum_{i=1}^{k} min\{\alpha_i^2,4\delta^2\},\\
    &= \sum_{i=1}^{k} min\{\alpha_i^2,4\beta\sigma^2\log(n)\},\\
    &\le max\{1,4\beta\log(n)\} \sum_{i=1}^{k} min\{\alpha_i^2,\sigma^2\},\\
    &\le (1+4\beta\log(n)) \sup_{\alpha \in A}\sum_{i=1}^{k} min\{\alpha_i^2,\sigma^2\},\\
    &\le 4\beta(1+\log(n))\:2.22\:\inf_{\hat{\alpha}}\sup_{\alpha\in A} E[\|\hat{\alpha}-\alpha\|^2],
\end{align}
where the last line follows from lemma \ref{lemma:dj2}.
\end{proof}
It can be shown that wavelet coefficients of functions residing in the TV class is solid and orthosymmetric. As shown in lemma \ref{lemm:subgauss}, the noisy wavelet coefficients are marginally sub-gaussian. Thus in the coefficient space, we are under the same observation model as in lemma \ref{lemma:dj2}. Using a uniform bound argument across all $O(n^2)$ bins and lemma \ref{lemm:subgauss} leads to the following corollary.
\begin{corollary}
\label{cor:minimax}
The soft-thresholded wavelet coefficients of re-centered and zero padded noisy data within any interval $[t_h,t_l]$ satisfy relation \eqref{eqn:minimax} with probability atleast $1-2n^{3-\beta/8}$.
\end{corollary}

Next, we record an important preliminary bound that will be used in proving the main result.
\begin{lemma}
\label{unifbound}
With probability at-least $1-\frac{\delta}{2}$, the total squared error for online averaging between two arbitrarily chosen time points $t_h$ and $t_l$ satisfies
\begin{equation}
\label{eqn:binbound}
    \sum_{t=t_h}^{t_l} (x_t^{t_h} - \theta_t)^2 \le 4\sigma^2\log(4n^3/\delta)(2+\log(t_l-t_h+1)) + 2(\theta_{t_h-1} - \theta_{t_h})^2 + 2\sum_{t=t_h+1}^{t_l}(\bar{\theta}_{t_h:t-1} - \theta_t)^2.
\end{equation}
\end{lemma}
\begin{proof}
Throughout this lemma we assume the notation $\theta_0 = 0$. For proving this, first we bound the squared error for online sample averages within a bin, $b[\ubar{l},\bar{l}]$, that starts and ends at fixed times $\ubar{l}$ and $\bar{l}$ respectively. Then a uniform bound argument will be used for bounding the squared error within any arbitrarily chosen bin. Note that $b[\ubar{l},\bar{l}]$ represents any fixed time interval and may not be even chosen by the policy. For $t \in [\ubar{l},\bar{l}]$, consider the prediction $x_t^{\ubar{l}}$, with same notation as used in the policy. Define a random variable $Z_t$ as
\begin{equation}
    Z_t =  \frac{(x_t^{\ubar{l}} - \theta_t)-(\lambda_t - \theta_t)}{\sigma\sqrt{1/[t-\ubar{l}]_{1+}}},
\end{equation}
where $[x]_{1+} = max\{1,x\}$, $\lambda_{l} = \theta_{\ubar{l}-1}$ and $\lambda_t = \bar{\theta}_{\ubar{l}:t-1}, \forall t > \ubar{l}$. $Z_t$ is subgaussian with variance parameter 1 and mean 0. Hence by sub-gaussian tail inequality, we have $P(|Z_t| \ge \sqrt{2\log(4/\delta)}) \le \delta/2$. By noting that length of a bin is $O(n)$ and applying uniform bound across all time points within the current bin we have
\begin{equation}
    P\left(\sup_{\ubar{l} \le t \le \bar{l}}|Z_t| \ge \sqrt{2\log(4n^3/\delta)}\right) \le \delta/2n^2.
\end{equation}
Hence with probability at-least $1-\delta/2n^2$, 
\begin{equation}
\label{timesup}
|x_t^{\ubar{l}} - \theta_t| \le |\lambda_t - \theta_t| + \sigma \sqrt{\frac{2\log(4n^3/\delta)}{[t-\ubar{l}]_{1+}}},\forall t \in [\ubar{l},\bar{l}].
\end{equation}
So the squared error within a bin can be bounded in probability as
\begin{equation}
    \sum_{t=\ubar{l}}^{\bar{l}} (x_t^{\ubar{l}} - \theta_t)^2 \le  2(\theta_{\ubar{l}-1} - \theta_{\ubar{l}})^2 + 2\sum_{t=\ubar{l}+1}^{\bar{l}}(\bar{\theta}_{\ubar{l}:t-1} - \theta_t)^2 + 2\sum_{t=\ubar{l}}^{\bar{l}}\sigma^2 \frac{2\log(4n^3/\delta)}{[t-\ubar{l}]_{1+}}.
\end{equation}
Here we applied the inequality $(a+b)^2 \le 2a^2+2b^2$ on \eqref{timesup}. Ultimately we are interested in analyzing the MSE within a bin detected by the policy. However the observations within a bin satisfies the restarting criterion of the policy and cannot be regarded independent. To break free of this constraint, we uniformly bound the quantity of interest --- MSE here --- across all possible bins. Noting that number of bins is $O(n^2)$ and applying uniform bound across all bins gives the following single sided tail bound.

Let E denote the event:\\
$\sup_{b[\ubar{l}:\bar{l}]}(x_t^{\ubar{l}} - \theta_t)^2 - 2(\theta_{\ubar{l}-1} - \theta_{\ubar{l}})^2 - 2\sum_{t=\ubar{l}+1}^{\bar{l}}(\bar{\theta}_{\ubar{l}:t-1} - \theta_t)^2 - 2 \sum_{t=\ubar{l}}^{\bar{l}}\sigma^2 \frac{2\log(4n^3/\delta)}{[t-\ubar{l}]_{1+}} \ge 0$.

Then,
\begin{align}
    P(E)
    &\le \delta/2.
\end{align}
Hence with probability at-least $1-\delta/2$, any bin $b[t_h:t_l]$ satisfies \eqref{eqn:binbound}.
\end{proof}

Since \eqref{eqn:binbound} holds for any arbitrary interval of the time  axis, it is particularly true for the bins discovered by the policy. Therefore the total squared error $T$ of the policy is upper bounded in probability by the sum of bin bounds of the form,
{\small
\begin{equation}
    \label{eqn:tse}
    T \le \sum_{m=1}^{M}4\sigma^2\log(4n^3/\delta)(2+\log(t_{l}^{(m)}-t_{h}^{(m)}+1))+ 2(\theta_{t_{h}^{(m)} - 1} - \theta_{t_{h}^{(m)}})^2 +  2\sum_{t=t_{h}^{(m)}+1}^{t_{l}^{(m)}}(\bar{\theta}_{t_{h}^{(m)}:t-1} - \theta_t)^2,
\end{equation}
}
where the outer sum iterates over the bins and $M$ is the total number of bins. The first term inside the outer summation can be controlled if we can upper bound $M$. Now we set out to prove our main theorem. 

\section{Proof of Theorem 1}
\label{proof_main_thm}

From the discussion in section \ref{sec:problem_setup}, the goal of bounding dynamic regret of the policy can be achieved by bounding the total squared error of its predictions. Our solution proceeds in two steps. First we upper bound the total squared error within a bin. Then we construct an upper bound for the number of bins spawned by the policy. With these two bounds in place, we bound the total squared error of the policy \eqref{eqn:tse}. 

Let's first proceed to get a bound on the last summation term in \eqref{eqn:tse}. We use a reduction towards Follow The Leader (FTL) strategy. The term is basically the regret incurred by an FTL game with quadratic loss function for the duration $[t_h,t_l]$.

Let $\Theta(t_h:t_l-1) = pad_0(\theta_{t_h},...,\theta_{t_l-1}) = [\Theta_{t_h},...,\Theta_{t_h+k-1}]^T$ denotes mean subtracted the zero padded true sequence in the interval $[t_h,t_l-1]$. Then,
\begin{align}
    \sum_{t=t_h}^{t_l}(\bar{\theta}_{t_h:t-1} - \theta_t)^2
    &= (\bar{\theta}_{t_h:t_l-1} - \theta_{t_l})^2 + \sum_{t=t_h}^{t_l-1}(\bar{\theta}_{t_h:t-1} - \theta_t)^2, \label{eqn:a1}\\
    &\le (\bar{\theta}_{t_h:t_l-1} - \theta_{t_l})^2 + \sum_{t=t_h}^{t_l-1}\frac{(\bar{\theta}_{t_h:t-1} - \theta_t)^2}{(t-t_h+1)} + \sum_{t=t_h}^{t_l-1}(\bar{\theta}_{t_h:t_l-1} - \theta_t)^2, \\
    &= (\bar{\theta}_{t_h:t_l-1} - \theta_{t_l})^2 + \sum_{t=t_h}^{t_l-1}\frac{(\bar{\theta}_{t_h:t-1} - \theta_t)^2}{(t-t_h+1)} + \|\Theta(t_h:t_l-1)\|^2\label{eqn:a2}.
\end{align}
We have applied FTL reduction for online game of predicting the true sequence $\theta_{t_h},...,\theta_{t_l-1}$ to get \eqref{eqn:a2}.

In the discussion below we assume that $\|D\theta_{1:n}\|_1 \le C_n$ and $|\theta_1| \le U$.

Now let's try to bound the term $\|\Theta(t_h:t_l-1)\|_2^2$. This is basically the regret of the best expert. By triangle inequality, 
\begin{align}
\|\Theta(t_h:t_l-1)\|^2
&\le \|\hat{\alpha}(t_h:t_l-1)\|_1^2 + \|\hat{\alpha}(t_h:t_l-1) - \alpha(t_h:t_l-1)\|_2^2,\\
&\le \left(\sum_{l=0}^{\log_2 (p) - 1} 2^{l/2} \|\hat{\alpha}(t_h:t_l-1)[l]\|_1   \right)^2 \notag \\
&\qquad + \|\hat{\alpha}(t_h:t_l-1) - \alpha(t_h:t_l-1)\|_2^2,\label{eqn:a6}
\end{align}

where $p$ is the padded length.

We can base our online averaging restart rule on the output of wavelet smoother. Suppose we decide to restart when $\|\hat{\alpha}(t_h:t_l)\|_1 \ge K n^{-1/3}$ for a constant $K$. Then the first term of \eqref{eqn:a6} gives the optimal rate of $O(n^{1/3})$ when summed across all bins. But the estimation error term $\|\hat{\alpha}(t_h:t_l-1) - \Theta(t_h:t_l-1)\|^2$ should also be controlled. If the smoother is minimax over any bin $[t_h,t_l]$, then we can hope to get minimaxity over the entire horizon. However, the total variation inside the bin is not known to the smoother. This is where the adaptive minimaxity of wavelet smoother comes to rescue.

Suppose $\cF$ denotes the class of functions $f$ with total variation $TV(f)\le C_n$. Let $\cA$ denote the set of all coefficients of the continuous wavelet transform of functions $f \in \cF$. Then $\cA \subset \Theta_{1,\infty}^{1/2}(C_n)$, where $\Theta_{1,\infty}^{1/2}(C_n)$ is a Besov body as defined in \citet{donoho1998minimax}. The minimax rate of estimation in this Besov body is $O(n^{-2/3}C_n^{2/3}\sigma^{4/3})$ where n is the number of observations. However, this is the rate of convergence of the $L_2$ function norm instead of the discrete (input-averaged) norm that we consider here. Over the Besov spaces, these two norms are close enough that the rates do not change (see section 15.5 of \citet{DJBook}). Hence Corollary \ref{cor:minimax} can be used to control the bias. 

Let $\hat{y}(t_h:t)$ denotes the soft-thresholding estimates of the vector $pad_0(y_{t_h:t})$.\\
i.e $\hat{y}(t_h:t) = H^T T(H \: pad_0(y(t_h:t)))$.

\begin{align}
    (\bar{\theta}_{t_h:t_l-1} - \theta_{t_l})^2
    &\le 2(\theta_{t_l-1} - \theta_{t_l})^2 + 2(\bar{\theta}_{t_h:t_l-1} - \theta_{t_l-1})^2,\\
    &\le 2(\theta_{t_l-1} - \theta_{t_l})^2 + 4(\hat{y}(t_h:t_l-1)[t_l-1] - (\bar{\theta}_{t_h:t_l-1} - \theta_{t_l-1}))^2 \notag \\
    &\quad + 4(\hat{y}(t_h:t_l-1)[t_l-1])^2.
    \label{eqn:outlier}
\end{align}

\begin{comment}
To bound the last term above first note that for any vector $\boldsymbol{v}$ whose length is a power of 2, we have $\|D\boldsymbol{v}\|_1 \le 2 \frac{1}{\sqrt{|\boldsymbol{v}|}} \sum_{l=0}^{\log_2(|\boldsymbol{v}|) - 1} 2^{l/2} \|\hat{\alpha}(t_h:t)[l] \|_1$, where $\alpha$ are the corresponding Haar coefficients. This can be seen by writing the wavelet expansion and applying the operator $D$ followed by triangle inequality. So if $\boldsymbol{v}$ denote the re-centered and padded observations within a bin before restart, then $\|D\boldsymbol{v}\|_1 \le \frac{2\sigma}{|\boldsymbol{v}|} \le 2\sigma$.
\end{comment}

Since L1 norm is greater than L2 norm, the policy's restart rule implies that

\begin{align}
    (\hat{y}(t_h:t_l-1)[t_l-1])^2
    &\le \sigma^2\label{eqn:lastTerm}
\end{align}

Combining \eqref{eqn:outlier} and \eqref{eqn:lastTerm}, we get

\begin{align}
    (\bar{\theta}_{t_h:t_l-1} - \theta_{t_l})^2
    &\le 2(\theta_{t_l} - \theta_{t_l-1})^2 + \gamma_1 (t_l-t_h)^{1/3}\:TV^{2/3}[t_h:t_l]\:\sigma^{4/3} +\sigma^2, \label{eqn:firstTerm}
\end{align}

where last line holds with probablity atleast $1- 2n^{3-\beta/8}$ due to Corollary \ref{cor:minimax}. Here $\gamma_1$ is a constant which can depend logarithmically on the width $t_l - t_h$.

Now let's bound the second term in \eqref{eqn:a2}. For any $t \in [t_h,t_l-1]$ we have,

\begin{align}
    \sum_{t=t_h}^{t_l-1}\frac{(\bar{\theta}_{t_h:t-1} - \theta_t)^2}{(t-t_h+1)}
    &\le \sum_{t=t_h}^{t_l-1} \frac{2(\theta_t - \theta_{t-1})^2 + 2(\bar{\theta}_{t_h:t-1} - \theta_{t-1})^2}{t-t_h+1},\\
    &\le \sum_{t=t_h}^{t_l-1} 2(\theta_t - \theta_{t-1})^2 \notag \\
    &\quad + \sum_{t=t_h}^{t_l-1} \frac{4(\hat{y}(t_h:t-1)[t-1] - (\bar{\theta}_{t_h:t-1} - \theta_{t-1}))^2 + 4(\hat{y}(t_h:t-1)[t-1])^2}{t-t_h+1},\\
    &\le \sum_{t=t_h}^{t_l-1} 2(\theta_t - \theta_{t-1})^2 + (\gamma_2 (t_l-t_h)^{1/3}\:TV^{2/3}[t_h:t_l]\:\sigma^{4/3} + 4 \sigma^2) (1 + \log n),\label{eqn:logterm}
\end{align}

where the last line holds with probability at-least $1- 2n^{3-\beta/8}$. 

\begin{align}
\|\Theta(t_h:t_l-1)\|_2^2
&\le\left(\sum_{l=0}^{\log_2 (p) - 1} 2^{l/2} \|\hat{\alpha}(t_h:t_l-1)[l]\|_1   \right)^2, \notag \\
&\qquad +  \gamma_3 (t_l-t_h)^{1/3}\:TV^{2/3}[t_h:t_l]\:\sigma^{4/3},\\
&\le \sigma^2 + \gamma_3 (t_l-t_h)^{1/3}\:TV^{2/3}[t_h:t_l]\:\sigma^{4/3}
\label{eqn:a7},
\end{align}
 with probability at-least $1- 2n^{3-\beta/8}$ for some constant $\gamma_3$ which can depend logarithmically on the width $t_l - t_h$.
 
 Due to Corollary \ref{cor:minimax} the bounds \eqref{eqn:firstTerm}, \eqref{eqn:logterm}, \eqref{eqn:a7} all simultaneously holds with probability at-least $1- 2n^{3-\beta/8}$. Combining these bounds, we get
 
 \begin{align}
     \sum_{t=t_h}^{t_l}(\bar{\theta}_{t_h:t-1} - \theta_t)^2
     &\le 2\|D\theta_{t_h:t_l}\|_2^2+ \gamma (t_l-t_h)^{1/3}\:TV^{2/3}[t_h:t_l]\:\sigma^{4/3} + 6\sigma^2 (1+\log(n)),
 \end{align}
 
 with probability at-least $1- 2n^{3-\beta/8}$ and $\gamma = \gamma_1 + \gamma_2 (1+\log(n)) + \gamma_3$.

 When summed across all bins as in \eqref{eqn:tse}, with probability at-least  $1-2n^{3-\beta/8}$ we have,

\begin{align}
    \sum_{m=1}^{M}\sum_{t=t_{h}^{(m)}}^{t_{l}^{(m)}}(\bar{\theta}_{t_{h}^{(m)}:t-1} - \theta_t)^2 &\le  U^2 + 2\|D\theta_{1:n}\|_2^2 + 6M \sigma^2 (1+\log n) \notag \\[-15pt]
    &\quad +\sum_{m=1}^{M} \gamma \: (k^{(m)})^{1/3}\:TV^{2/3}[t_h^{(m)}:t_l^{(m)}]\:\sigma^{4/3}\label{eqn:10}, \\[10pt]
    &\le U^2 + 2\|D\theta_{1:n}\|_2^2 + 6M \sigma^2 (1+\log n) \notag \\
    &\qquad + \gamma\sigma^{4/3}n^{1/3}\left(\sum_{m=1}^{M}\frac{k^{(m)}}{n}\right)^{\frac{1}{3}}\left(\sum_{m=1}^{M}TV[t_h^{(m)}:t_l^{(m)}]\right)^{\frac{2}{3}},\label{eqn:11}\\[10pt]
    &\le U^2 +2\|D\theta_{1:n}\|_2^2 + 6M \sigma^2 (1+\log n) \notag \\
    &\qquad + 2\gamma\sigma^{4/3}n^{1/3}C_n^{2/3}\label{eqn:12}.
\end{align}

Here $k^{(m)}$ is the length of $\Theta(t_h^{(m)}:t_l^{(m)}-1)$. The term $(\theta_{t_h^{(m)}-1} - \theta_{t_h^{(m)}})^2$ is at-most $U^2$ for the first bin. We arrive at \eqref{eqn:11} by applying Holder's inequality $x^Ty \le \|x\|_p\|y\|_q$ with $p=3$ and $q=3/2$. For both  \eqref{eqn:11} and \eqref{eqn:12} we use the fact that $\sum_{m=1}^{M}k^{(m)} \le 2n$ where the factor 2 is an artifact of zero-padding.

By appealing to lemma \ref{lemma:binbound}, we have with probability at-least $1-4n^{3-\beta/8}$,

\begin{align}
    \sum_{m=1}^{M}\sum_{t=t_{h}^{(m)}}^{t_{l}^{(m)}}(\bar{\theta}_{t_{h}^{(m)}:t-1} - \theta_t)^2
    &\le U^2 +2\|D\theta_{1:n}\|_2^2 + 12\sigma^2 \log n \notag \\
    &\qquad + 24 (\log(n))^2 n^{1/3} C_n^{2/3}\sigma^{4/3} + \gamma\sigma^{4/3}n^{1/3}C_n^{2/3}.\label{eqn:13}
\end{align}

Next, we proceed to bound the first summation terms in \eqref{eqn:tse}. For this, we upperbound the number of bins to control the concentration terms in \eqref{eqn:tse} when summed across all bins. Essentially our decision rule should not lead to over binning. Observe that the sum of total variations across all bins is $C_n$. If the decision rule guarantees (at-least in probability) that total variation inside any detected bin is $\tilde{\Omega}(n^{-1/3})$, then the number of bins is optimally $O(n^{1/3})$. Such a TV lower bounding property is satisfied by wavelet soft-thresholding as described in lemma \ref{lemma:binbound}. This is facilitated by the uniform shrinkage property of soft-thresholding estimator. More precisely,

Let's denote
\begin{align}
    V_m &= 4\sigma^2\log(2n^3/\delta)(2+\log(t_{l}^{(m)}-t_{h}^{(m)}+1)).
\end{align}
Then,
\begin{align}
     \sum_{m=1}^{M} V_m &\le 4\sigma^2\log(4n^3/\delta)(2+\log(n)) \max\{1,2n^{1/3}C_n^{2/3}\sigma^{-2/3}\log(n)\}, \\
    &\le 4\sigma^2\log(4n^3/\delta)(2+\log(n)) \notag \\
    &\qquad + 8n^{1/3}C_n^{2/3}\sigma^{4/3}\log(n) \log(4n^3/\delta)(2+\log(n))\label{eqn:firstbound},
\end{align}

with probability at-least $1-2n^{3-\beta/8}$. Here $[t^{m}_h, t^{m}_l]$ corresponds to the $m^{th}$ bin discovered by the policy. This relation follows due to Lemma \ref{lemma:binbound}.

Combining \eqref{eqn:firstbound} and \eqref{eqn:13} we have with probability at-least $1-4n^{3-\beta/8}-\delta/2$
\begin{align}
    \begin{split}
    \label{eqn:fin_inter}
    T \le{}&
    8n^{1/3}C_n^{2/3}\sigma^{4/3}(2+\log(n))\log(n) \\
    &+ 4\sigma^2\log(4n^3/\delta)(2+\log(n))\\
    &+ U^2 +2\|D\theta_{1:n}\|_2^2 + 12\sigma^2 \log n \\
    &\qquad + 24 (\log(n))^2 n^{1/3} C_n^{2/3}\sigma^{4/3} + 2\gamma\sigma^{4/3}n^{1/3}C_n^{2/3}.
    \end{split}
\end{align}

By observing that $\|D\theta_{1:n}\|_2 \le \|D\theta_{1:n}\|_1 = C_n$ we get the bound,
\begin{align}
    \label{finalbound}
    \begin{split}
    T \le{}&
    8n^{1/3}C_n^{2/3}\sigma^{4/3}(2+\log(n))\log(n) \\
    &+ 4\sigma^2\log(4n^3/\delta)(2+\log(n))\\
    &+ U^2 +2C_n^2 + 12\sigma^2 \log n \\
    &\qquad + 24 (\log(n))^2 n^{1/3} C_n^{2/3}\sigma^{4/3} + 2\gamma\sigma^{4/3}n^{1/3}C_n^{2/3}.
    \end{split}
\end{align}
The above bounds holds with probability at-least $1-\delta$, if we set $\beta = 24+\frac{8\log(8/\delta)}{\log(n)}$.

We conclude our proof by observing that the above arguments can be readily extended to any batch smoother that satisfy the following criteria.
\begin{description}[font=$\bullet$\scshape\bfseries]
\item Adaptive minimaxity over any interval within the time horizon.
\item The restart decision rule optimally lowerbounds the total variation of any spawned bin.
\end{description}
Thus our policy can be viewed as a meta-algorithm that lifts a ``well behaved smoother'' to an optimal forecaster in the online setting.

Next we remark how the proof can be adapted to the setting where an extra boundedness constraint is put on the ground truth. i.e, $\theta_{1:n} \in TV(C_n)$ and $|\theta_i| \le B, i=1,\ldots,n$. Then the $U^2$ term in \eqref{eqn:fin_inter} becomes $B^2$. The additive $\|D\theta_{1:n}\|_2^2$ term can be bounded as,

\begin{align}
    \|D\theta_{1:n}\|_2^2
    &= \sum_{i=2}^{n} (\theta_i - \theta_{i-1})^2,\\
    &\le \sum_{i=2}^{n} (|\theta_i| + |\theta_{i-1}|)(|\theta_i - \theta_{i-1}|),\\
    &\le 2BC_n.
\end{align}

Thus when $\|\theta_{1:n}\|_\infty \le B$ and if we set $\beta=24+\frac{8\log(8/\delta)}{\log(n)}$ then with probability at-least $1-\delta$,

\begin{align}
    \label{finalbound_constrainted}
    \begin{split}
    T \le{}&
    8n^{1/3}C_n^{2/3}\sigma^{4/3}(2+\log(n))\log(n) \\
    &+ 4\sigma^2\log(4n^3/\delta)(2+\log(n))\\
    &+ B^2 +4BC_n + 12\sigma^2 \log n \\
    &\qquad + 24 (\log(n))^2 n^{1/3} C_n^{2/3}\sigma^{4/3} + 2\gamma\sigma^{4/3}n^{1/3}C_n^{2/3}.
    \end{split}
\end{align}
\section{Adaptive Optimality in Discrete Sobolev class} \label{app:sobolev}

In this section, we establish that despite the fact that \ARROWS{} is designed for the total variation class, it adapts to the optimal rates forecasting sequences that are more regular.

The discrete Sobelov class is defined as
\begin{align}
    \cS(C_n') = \{\theta_{1:n} : \|D \theta_{1:n}\|_2 \le C_n' \}.
\end{align}
The minimax cumulative error of nonparametric estimation in the discrete Sobolev class is $\theta_{1:n}(n^{2/3}[C'_n]^{2/3}\sigma^{4/3})$ \citep[see e.g.,][Theorem 5 and 6]{sadhanala2016total}. 

Recall that the discrete Total Variation class that we considered in this paper is defined as 
\begin{align}
    \cT(C_n) = \{\theta_{1:n} : \|D \theta_{1:n}\|_1 \le C_n \}.
\end{align}

By the norm inequalities, we know that
$$
\cT(C_n') \subset  \cS(C_n')  \subset \cT(C_n'\sqrt{n}).
$$

The following refinement of our main theorem establishes that \ARROWS{} also achieves the minimax rate in discrete Sobolev classes.

\begin{theorem}
    \label{thm:main_sobelov}
    Let the feedback be $y_t = \theta_t + Z_t$ where $Z_t$ is an independent, $\sigma$-subgaussian random variable. 
    Let $\theta_{1:n} \in \cS(C_n')$. If $\beta = 24+\frac{8\log(8/\delta)}{\log(n)}$, then with probability at least $1-\delta$, \ARROWS{} achieves 
    a dynamic regret of $\tilde{O}(n^{2/3}[C_n']^{2/3}\sigma^{4/3} +U^2 +  [C_n^']^2 + \sigma^2 )$ where $\tilde{O}$ hides a logarithmic factor in $n$ and $1/\delta$.
    %The dynamic regret of the policy is $\tilde{O}_p(n^{1/3}C_n^{2/3}\sigma^{4/3} + )$
\end{theorem}
\begin{proof}
Let's minimally expand the Sobolev ball to a TV ball of radius $C_n = \sqrt{n}C_n'$. This chosen radius of the TV ball is in accordance with the canonical scaling introduced in \citep{sadhanala2016total}. This activates the following embedding:
\begin{align}
    \cS_1(C_n') \subseteq TV(C_n).
\end{align}

We can rewrite \eqref{eqn:fin_inter} as
\begin{align}
    \label{finalbound_sob}
    \begin{split}
    T \le{}&
   8n^{1/3}\|D\theta_{1:n}\|_1^{2/3}\sigma^{4/3}(2+\log(n))\log(n) \\
    &+ 4\sigma^2\log(4n^3/\delta)(2+\log(n))\\
    &+ U^2 +2\|D\theta_{1:n}\|_2^2 + 12\sigma^2 \log n \\
    &\qquad + 24 (\log(n))^2 n^{1/3} \|D\theta_{1:n}\|_1^{2/3}\sigma^{4/3} + 2\gamma\sigma^{4/3}n^{1/3}\|D\theta_{1:n}\|_1^{2/3}.
    \end{split}
\end{align}

The above representation reveals the optimality of our policy over Sobolev class $S_1(C_n')$. Enlarging the Sobolev class to the TV class that contains it does not change the minimax rate in the smoothing setting.  See, e.g., Theorem 5 and 6 of \citep{sadhanala2016total} and take $d = 1$, and $C'_n = n^{-1/2} C_n$. By using $\|x\|_1 \le n^{1/2}\|x\|_2$ for $x \in \mathbb{R}^n$,
\begin{align}
    \frac{\|D\theta_{1:n}\|_1}{n^{1/2}} \le \|D\theta_{1:n}\|_2 \le C'_n = \frac{C_n}{n^{1/2}}.
\end{align}
Plugging this bound on $\|D\theta_{1:n}\|_1$ in \eqref{finalbound_sob} recovers the minimax regret for the Sobolev class of radius $C'_n$. The additional term of $\|D\theta_{1:n}\|_2^2$ --- similar to as shown in in appendix \ref{add_proofs} --- is unavoidable in the online setting for predicting discrete Sobolev sequences.

\end{proof}

\begin{remark} \label{rem:sob}
    Note that $\cT(C_n')\subset \cS(C_n')$, therefore our lower bound from Proposition~\ref{prop:lowerbound} still applies, which suggests that the additional $[C_n^']^2+\sigma^2$ is required and that \ARROWS{} is an optimal forecaster for sequences in Sobolev classes as well.
\end{remark}

\section{Fast Computation}\label{sec:runtime}
We describe the proof of $O(n\log n)$ runtime guarantee below.

	We use an inductive argument. Without loss of generality let the start of current bin be at time 1. Suppose we know the wavelet transform of points upto time $t$. Let the next highest power of 2 for both $t$ and $t+1$ be $p$. We identify this value as a pivot for time $t$ and $t+1$. Zero padding is done to hit this pivot. We can view the $pad_0$ operation at time $t+1$ as the difference between the padded original data and and a step signal. This step signal assume the value $\bar{y}_{1:t+1}$ in time $[1,t+1]$ and 0 in $[t+2,p]$. For computing wavelet transform of the step, we need to update only $O(\log(p))$ coefficients. Inputs to the Haar transform of the padded data at times $t$ and $t+1$ differs by just one co-ordinate. Hence coefficients of only  $\log(p)$ wavelets need to be changed. Each such change can be performed in $O(1)$ time in an incremental fashion. 
	
	Now let's consider the case when the pivot for time $t+1$ is $2t$. Suppose we know the Haar wavelet coefficients upto time $t$. In this case, we need to compute the coefficients of $\log(t)$ newly introduced wavelets that span the interval $[t,2t]$ since the zero padding will force most of the new wavelet coefficients to be zero. The computation of each of those new coefficients can be done in $O(1)$ due to sparsity of signal in interval $[t,2t]$. We also need to change the first two wavelet coefficients which can be done again in in $O(1)$ time. In all these cases, we only need to do soft-thresholding to the newly updated coefficients. At the base case, when the pivot is just 2, then the computation can be in $O(1)$ time. Thus within a pivot $p$, the number of computations required is $O(p\log(p))$ which translates to $O(k^{(m)}\log(k^{(m)}))$ computations within the  $m^{th}$ bin. Summing across all the bins yields a runtime complexity of $O(n\log(n))$. 

\section{Regret of AOMD}\label{sec:regret_aomd}
In this section we prove that for any predictable sequence $\{M_t\}_{t=1}^{n}$, the AOMD algorithm has a dynamic regret of $\tilde{O}(\sqrt{n})$ when applied to our problem. As discussed in Section~\ref{sec:related_work}, consider loss functions $f_t(x) = (x - y_t)^2$ and comparator sequence $\{u_t\}_{t=1}^{n}$. First let's consider a deterministic noise setting \citep{softThreshold95}:
\begin{equation}
    y_t = \theta_t + \delta\:\sigma\sqrt{20\log(n)},
\end{equation}
where $|\delta| \le 1$ is chosen by a clever adversary. Let's proceed to get a bound on the quantity $D_n$. The gradient of our loss function is $2(x-y_t)$. So after observing the values of $x_t$ and $M_t$, an adversary can pick a suitable $\delta$ such that each term of $D_n$ 
\begin{equation}
\label{eqn:Dn}
D_n = \sum_{t=1}^{n}\|\nabla f_t(x_t) - M_t \|_{*}^2.
\end{equation}
can be made $O(1)$. This gives an $O(n)$ bound for $D_n$. 

We can show that $V_n$ is $O(n)$ if we assume that $\cX$ is compact and all of the $y_t$ is bounded. Boundedness of $y_t$ follows from the assumptions (A3) and (A4). By appealing to assumption (A3) we see that
\begin{equation}
C_n(u_1,u_2,...,u_n) = \sum_{t=1}^{n}\|u_t - u_{t-1}\|.
\end{equation}
$C_n(\theta_1,...,\theta_n)$ is $O(1)$. Plugging this into the regret bound specified in \citet{jadbabaie2015online} bounds the dynamic regret in our setting as $\tilde{O}(\sqrt{n})$.

We now relate this deterministic noise setting to the guassian setting where the observations are produced according to $y_t = \theta_t + Z_t$, where $Z_t$ is a zero mean sub gaussian with parameter $\sigma^2$. As described in proof of theorem \ref{thm:dj}, $P(\sup_{i}|Z_i| \ge \sigma\sqrt{20\log(n)}) \le 2n^{-9}$. Hence by conditioning on the event that $\sup_{i}|Z_i| \le \sigma\sqrt{20\log(n)}$, the regret bound of the deterministic noise setting applies to gaussian setting with high probability.

\section{Lower bound proof}\label{app:lower}
\label{add_proofs}

\begin{proof}[Proof of Proposition~\ref{prop:lowerbound}]
	First, a lower bound of $\Omega(n^{1/3} C_n^{2/3}\sigma^{4/3})$ is given by \citep{donoho1998minimax} for the smoothing estimator $x_{1:n}$ that has more information than we do. 
The argument uses the fact that the TV-ball is sandwiched between two Besov-bodies with identical minimax rate. To the best of our knowledge, the dependence on $C_n$ and $\sigma$ is first made explicit in, e.g., \citep{birge2001gaussian}.
		
By the fact that  ``the max is larger than the mean'', we have that for any prior distribution $\cP$,
	\begin{align*}
	\sup_{\theta_{1:n}\in \TV(C_n)} \E\left[\sum_{t=1}^n (x_t - \theta_t)^2 \right]  \geq \E_{\theta_{1:n}\sim \cP}\left[\E[\sum_{t=1}^n (x_t - \theta_t)^2 | \theta_{1:n} ]\right].
	\end{align*}

Take $\cP$ such that 
\begin{enumerate}
	\item $\theta_1= U$ with probability $0.5$ and $-U$ otherwise.
	\item $\theta_2=\theta_1 + C_n$ with probability $0.5$ and $\theta_1-C_n$ otherwise.
	\item $\theta_t = \theta_2$ for $t = 3,4,...,n$.
\end{enumerate}

Note that $x_1$ does not observe anything yet, therefore $x_1=0$ is the Bayes optimal decision rule. This gives a trivial lower bound of $\E\left[(x_1-\theta_1)^2\right] \geq U^2$. Now, let's reveal $\theta_1$ to $x_2$ an additional information, then by the same argument, we have that $\E\left[ (x_2-\theta_2)^2 \right]\geq C_n^2$. 

Consider an alternative $\cP$ when $\theta_1 = ... = \theta_n =\theta$. 
Let the noise be iid Gaussian with variance $\sigma^2$.  In this case the problem reduces to a naive statistical estimation problem with $\theta \in [-U,U]$. For each $t$ which observes $t-1$ iid samples from $\cN(\theta,\sigma^2)$, then by \citet{bickel1981minimax}, the minimax risk for this problem is
$$
\inf_{\hat{\theta}}\sup_{\theta\in [-U,U]} \E(\hat{\theta}-\theta)^2 = \frac{\sigma^2}{t} - \frac{\pi^2\sigma^4}{tU^2}  +  o(\frac{\sigma^4}{tU^2}).
$$
Summing over $t=2,3,...,n$, and apply the upper/lower bounds of the harmonic series, we have a lower bound of 
$$
\E\left[\sum_{t=1}^n (x_t-\theta_t)^2 \right]\geq \max\{0, \sigma^2 \log(n+1)  - \frac{\pi^2\sigma^4}{U^2} (1+\log(n)) (1+ o(1))\}.
$$
Take the condition that $U > 2\pi\sigma$ and $n > 3$, the above expression can be further lower bounded by $0.5 \sigma^2 \log(n)$. Note that this bound applies even if $C_n=0$.

Finally, we can similarly apply the same argument to the case when $\theta_1 = 0$ and $\theta_2 = ... = \theta_n = \theta$ and where the constraint is that $-C_n \leq \theta \leq C_n$.
This gives us a lower bound of
$$
\E\left[\sum_{t=2}^n (x_t-\theta_t)^2 \right]\geq \max\{0, \sigma^2 \log(n)  - \frac{\pi^2\sigma^4}{C_n^2} (1+\log(n-1)) (1+ o(1))\}.
$$
If $C_n > 2\pi\sigma$ and $n > 3$, we can again bound it below by $0.5 \sigma^2 \log(n)$.  In other word, we get the $\sigma^2 \log(n)$ lower bound provided that either $C_n$ or $U$ is greater than $2\pi\sigma$.

The proof is complete by taking the average of lower bounds above. We can take $c=1/6$.
\end{proof}

\subsection{Lower bound with extra boundedness constraint on ground truth}
Suppose we assume $|\theta_i| \le B, i = 1,\ldots,n$. Then we can adapt the proof presented above by considering a prior $\cP$ such that $\theta_i= \epsilon_i B, i=1,\ldots,1+\floor{C_n/2B}$. $\theta_i = \theta_{1+\floor{C_n/2B}}, \forall i > 1+\floor{C_n/2B}$. Here $\epsilon_i$ are independent random variables assuming value $+1$ with probability 0.5 and $-1$ with probability 0.5. Assume that we reveal to learner the probability law of observations $\theta_i$. Under this setting we can see that $\E\left[\sum_{t=1}^n (x_t-\theta_t)^2\right] \geq B^2 + B C_n/2$.

\subsection{Connections to other lower bounds in literature}

\citep{besbes2015non} derived a lower bound of $O(n^{1/2} V_n^{1/2})$ by packing a sequence of quadratic loss functions. Note that this is larger than the upper bound that we attain with quadratic losses. Though this observation seems confusing, a careful study reveals that there is no contradiction. For constructing the lowerbound, \citep{besbes2015non} used a variational budget $V_n$ as , $V_n = \sum_{t=2}^{n} \sup_{x \in conv(\theta_1,...\theta_n)} |f_t(x) - f_{t-1}(x)|=\sum_{t=2}^{n} \sup_{x \in [\theta_{min},\theta_{max}]} |(x-\theta_t)^2 - (x-\theta_{t-1})^2|$, where $conv(.)$ denotes the convex hull of a sequence of points. This is different from the variational budget they use in section 2 of their paper and is also different from $C_n$ that we use for the TV class. When applied to our setting this $V_n$ is no longer proportional to our $C_n$, instead, it is proportional to $(\theta_{max} - \theta_{min})C_n$.

The packing set constructed through the functions defined in equation (A-12) of \citep{besbes2015non}  obeys $(\theta_{max} - \theta_{min}) =  \frac{1}{2} V_n^{1/4} n^{-1/4}$. So we have $C_n =\frac{V_n}{V_n^{1/4}n^{-1/4}} = V_n^{3/4}n^{1/4}$, where we have subsumed proportionality constants. Thus we see that $V_n = \frac{C_n^{4/3}}{n^{1/3}}$. Putting this into their lowerbound recovers exactly our $n^{1/3}C^{2/3}$ bound. 

The additional  $C_n^2$ term that appears in our upper bound is required for any methods that do online forecasting of sequences in the TV class.  The reason why OGD appears to not require $C_n^2$ according to \citep{besbes2015non} is because they require the $\theta_t$ to be bounded for all $t$, while we only require $\theta_1$ to be bounded by $U$ (see Theorem~\ref{thm:ogd}).

The lowerbound discussed in \citep{yang2016tracking} considers a more general setting of smooth non-strongly convex sequence of loss functions. Such a lowerbound will not apply in our more restrictive setting.

\section{Optimality of linear forecasters in Discrete Sobolev class}\label{app:sobexp}

In this section we first establish that just like \ARROWS{}, linear strategies such as OGD and MA are also optimal forecasters for sequences in Discrete Sobolev class. Then we substantiate it using experiments.
\begin{theorem} \label{thm:sob_ogd}
    Let the feedback be $y_t = \theta_t + Z_t$ where $Z_t$ is an independent, $\sigma$-subgaussian random variable. 
    Let $\theta_{1:n} \in \cS(C_n')$. Restarting OGD with batch size of $\frac{\sigma^{2/3}(n \log n)^{1/3}}{[C_n']^{2/3}}$ achieves an expected dynamic regret of $\tilde{O}(U^2 + [C_n']^2 + n^{2/3}[C_n']^{2/3}\sigma^{4/3})$.
\end{theorem}
\begin{proof}
We stick to the same notations as in Appendix \ref{gentle}. Let's start the analysis from \eqref{eqn:ogd_decomp}. Let $t' = t - t_h^{(i)}$.
\begin{align}
    (\theta_t - \bar{\theta}_{t_h^{(i)}:t-1})^2
    &\le \frac{\left(\sum_{i=t_h^{(i)}}^{t-1} (\theta_t - \theta_i)\right)^2}{[t']^2},\\
    &\le \frac{t'}{[t']^2} \sum_{i=t_h^{(i)}}^{t-1} (\theta_t - \theta_i)^2,\\
    &\lsim L [C'_i]^2.
\end{align}
Hence summing across all points yields,
\begin{align}
    R_i 
    &\lsim L^2 [C'_i]^2 + \sigma^2 \log L .
\end{align}
So the total expected regret becomes,
\begin{align}
    \sum_{i=1}^{\ceil{n/L}} R_i
    &\lsim L^2 [C'_n]^2 + \frac{n}{L} \sigma^2 \log L .
\end{align}
By choosing $L = \frac{\sigma^{2/3}(n \log n)^{1/3}}{[C_n']^{2/3}}$ we get the theorem. The additive term $[C_n']^2$ arises similarly as in proof of Theorem \ref{thm:ogd}
\end{proof}
The optimality of Moving Averages can be proved similarly.

\begin{remark}
Thus from Theorems ~\ref{thm:main}, ~\ref{thm:main_sobelov1}, ~\ref{thm:ogd}, ~\ref{thm:sob_ogd} we see that \ARROWS{} is minimax over both the classes $TV(C_n)$ and $\cS(C_n/\sqrt{n})$ while linear forecasters such as OGD and MA require different tuning parameters to perform optimally in each class.    
\end{remark}

Next, we give numerical experiments substantiating the claims.
\\
\\
\textbf{Experimental results:}
\begin{figure}[tbh]
\centering
\includegraphics[width=0.6\textwidth]{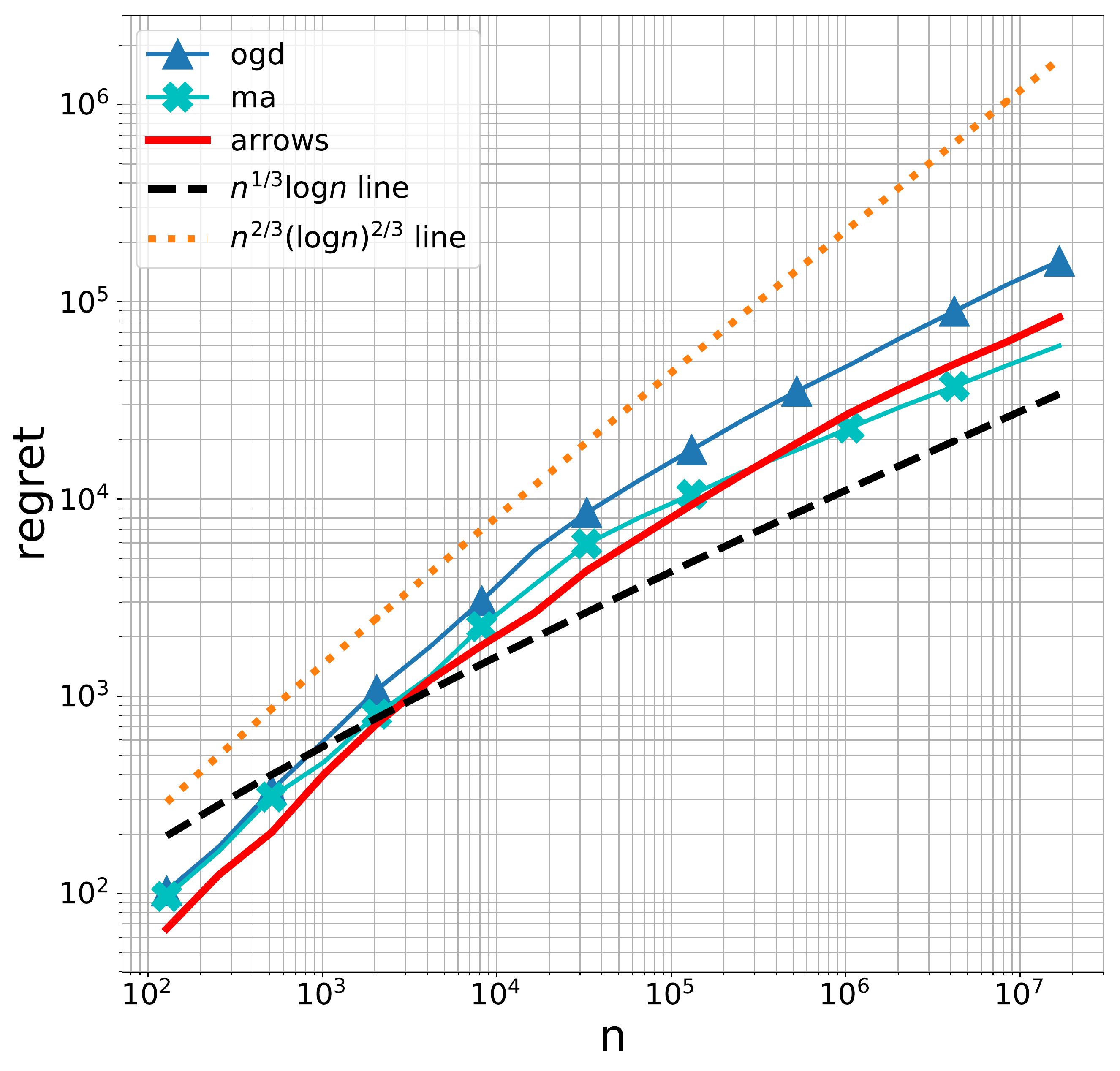}
\caption{Regret plot for policies calibrated according to Sobolev radius for a Doppler function} \label{fig:doppler_sob}
\end{figure}
Here we consider a doppler function $f(t) = \sin \left(\frac{2\pi(1+\epsilon)}{t/n+ 0.01}\right)$ with $n$ being the time horizon. For this function $C'_n = \|D\theta\|_2 = O(C_n/\sqrt{n})$ when $n$ is sufficiently large and $\|D\theta\|_2 = O(C_n)$ for small $n$ for a TV bound $C_n = O(1)$. Thus for sufficiently large $n$, this sequence belong to a  small Sobolev ball with radius $O(1/\sqrt{n})$ while the TV class that encloses that Sobolev ball as per Theorem \ref{thm:main_sobelov} has radius $O(1)$. 

We observe noisy data $y_i = f(i/n) + z_i$, $i=1,...,n$ and $z_i$ are iid normal variables with $\sigma = 1$.Figure \ref{fig:doppler_sob} plots the regret averaged across 5 runs in a log log scale. The necessary input calibration was made as per Remark \ref{rem:sob} while running \ARROWS{}. We can see that in this case all the algorithms perform in an optimal manner. 

Specifically we identify two regimes one for small $n$ and other for larger $n$. When $n$ is large, we obtain the minimax regret rate $\tilde O(n^{1/3})$ due to small $C'_n$ which can be considered as $O(1/\sqrt{n})$. Numerically for $n>10^5$, $C'_n$ is less than 0.1\% of $C_n$. For smaller values of $n$ where $C'_n$ can be not too small, we attain a regret in accordance with the $\tilde O(n^{2/3})$ minimax rate. Numerically when $n < 10^4$, $C'_n$ is atleast 8.5\% of $C_n$ which can be considered as $O(C_n) = O(1)$.

\end{document}